\documentclass[twoside]{article}
\usepackage[accepted]{mlmonitoring}

\makeatletter
\renewcommand{\Notice@String}{* Equal contribution. \vspace{0.7cm}\\
Preprint. Under review.}
\makeatother

\usepackage[utf8]{inputenc} % allow utf-8 input
\usepackage[T1]{fontenc}    % use 8-bit T1 fonts
\usepackage{url}            % simple URL typesetting
\usepackage{booktabs}       % professional-quality tables
\usepackage{amsfonts}       % blackboard math symbols
\usepackage{nicefrac}       % compact symbols for 1/2, etc.
\usepackage{xcolor}         % colors
\usepackage{amsmath}
\usepackage{amssymb}
\usepackage{graphicx}
\usepackage{amsthm}
\usepackage{bbm}
\usepackage{dsfont}
\usepackage{comment}
\usepackage{algorithm}
\usepackage{algpseudocode}

\newtheorem{prop}{Proposition}
\newtheorem{definition}{Definition}
\usepackage{thmtools}
\usepackage{thm-restate}
\usepackage{enumitem}
\usepackage{makecell}
\usepackage{caption}

\usepackage[round]{natbib}

\begin{document}

\twocolumn[

\customtitle{Label-Efficient Monitoring of Classification Models via Stratified Importance Sampling}

\customauthor{ Lupo Marsigli\textsuperscript{*} \And Angel Lopez de Haro\textsuperscript{*}}

\customaddress{ Amazon \And  Amazon %\AND \texttt{\{lupomars,angelglh\}@amazon.com}
} 
]

% Abstract
\begin{abstract}
Monitoring the performance of classification models in production is critical yet challenging due to strict labeling budgets, one-shot batch acquisition of labels and extremely low error rates. We propose a general framework based on Stratified Importance Sampling (SIS) that directly addresses these constraints in model monitoring. While SIS has previously been applied in specialized domains, our theoretical analysis establishes its broad applicability to the monitoring of classification models. Under mild conditions, SIS yields unbiased estimators with strict finite-sample mean squared error (MSE) improvements over both importance sampling (IS) and stratified random sampling (SRS). The framework does not rely on optimally defined proposal distributions or strata: even with noisy proxies and sub-optimal stratification, SIS can improve estimator efficiency compared to IS or SRS individually, though extreme proposal mismatch may limit these gains. Experiments across binary and multiclass tasks demonstrate consistent efficiency improvements under fixed label budgets, underscoring SIS as a principled, label-efficient, and operationally lightweight methodology for post-deployment model monitoring.
\end{abstract}

% Introduction
\section{INTRODUCTION} \label{sec:introduction}

In production machine learning (ML) systems it is critical to continuously monitor the predictive performance of deployed models \citep{ginart2022mldemon}. Models usually face conditions where even small drops in accuracy can have costly or unsafe consequences if left undetected, making reliable performance monitoring essential \citep{gama2014survey}.
Model monitoring in real-world scenarios has been approached either through label-free methods, which can surface potential distribution shifts or alarms \citep{rabanser2019failing, roschewitz2024automatic}, or through sampling-based designs. Our focus is on the latter: we assume access to an oracle that reveals the true label when queried, but the number of allowable queries is constrained by a finite budget, rendering exhaustive labeling infeasible. At the same time, real-world model monitoring is challenging as defect rates in stationary conditions are extremely low, which makes naive random sampling inefficient. The key challenge is to design label-efficient sampling methods in order to enable efficient statistical assessment of model performance under finite sample.

A large body of prior work has explored stratified random sampling (SRS) and importance sampling (IS) as label-efficient designs. 
Stratified sampling partitions the data into homogeneous groups, reducing estimator variance when strata are well-aligned with the target variable \citep{fogliato}. Importance sampling instead prioritizes instances with higher predicted error likelihood \citep{importance-sampling-paper}. Adaptive monitoring methods have also been proposed, updating sampling distributions online, but these often require assumptions for continuous re-estimation, which limits their practicality in batch annotation workflows and sacrifice the generality of a static design.
More broadly, monitoring research in ML has tended to focus on drift detection or heuristics and ad-hoc investigations on stratification and importance distributions, rather than principled, finite-sample variance reduction for metric estimation.

In this paper, we revisit stratified importance sampling (SIS), a classical variance-reduction method in Monte Carlo analysis \citep{owen2013monte}, and establish it as a general-purpose framework for model monitoring. While SIS is well-known in domains as risk management and reliability analysis \citep{kim2023stratified, xiao2020reliability}, its potential for classification monitoring under realistic post-deployment constraints has not been formalized. Our analysis is deliberately general: we do not prescribe specific strata definitions or importance weight functions. Instead, we show that even with noisy proxies, uncalibrated scores and sub-optimal strata, combining stratification and importance weighting yields systematic improvements over either SRS or IS alone. This provides a robust baseline for practitioners, with guarantees that hold under finite label budget that is operationally lightweight compared to adaptive methods.

\textbf{Problem statement.}
We consider the task of estimating classification performances under labeling constraints, in binary and multiclass settings. Annotations are collected in batches and not immediately used for model retraining. The monitoring objective is to estimate defect rates with minimal mean squared error given a fixed budget. It holds under real-world generality of unknown parametrics assumptions and distributions, accounting for potentially uncalibrated model scores and noisy proxies.
This work is motivated by high-impact applications of different kind, including breast cancer detection, fraud detection, rare disease diagnosis, and safety-critical classification tasks such as image recognition, where inaccurate estimations can have serious implications.

\textbf{Contributions and outline.} This paper makes three contributions. First, we formalize SIS as a framework for monitoring classification models under practical constraints, including one-shot batch labeling, uncalibrated model scores, and extremely low error rates. Second, we provide a finite-sample theoretical analysis showing that the SIS estimator achieves strictly lower mean squared error (MSE) than either importance sampling (IS) or stratified random sampling (SRS) under mild conditions. Our analysis also identifies scenarios in which SIS offers limited benefit, thereby delineating the scope of its applicability. Third, we validate the framework empirically through experiments on binary and multiclass tasks across both tabular and image datasets. The results demonstrate consistent efficiency improvements while maintaining robustness to noisy proxies and imperfect stratification.
Taken together, these contributions establish SIS as a principled, finite-sample–guaranteed framework for label-efficient model monitoring, providing a statistically rigorous baseline that complements adaptive and drift-detection approaches while remaining operationally lightweight.

The remainder of the paper is organized as follows. Section \ref{sec:problem-setup} formalizes the problem of estimating error signal $\epsilon$ under a fixed annotation budget and introduces the notation. Section \ref{sec:methods} adapts the SIS method to model monitoring and provide a theoretical analysis proving the increase in efficiency of SIS compared to standard static baselines. Section \ref{sec:results} presents experimental results in classification space, showing the practical gains in efficiency we achieved with SIS. Section \ref{sec:related-work} reviews the current state-of-the-art approaches in classification models monitoring. We conclude in Section \ref{sec:conclusion}.

% Problem Setup
\section{PRELIMINARY} \label{sec:problem-setup}

Consider a classifier $f$ trained and then evaluated on a hold-out test set. Upon successful evaluation, it is deployed to production where  it performs inference over a stream of instances \( (X_i)_{i \geq 1} \), where \( X_i \in \mathcal{X} \) and \( Y_i \in \mathcal{Y} = \{1, \dots, K\} \) is the unobserved label. For each \( X_i \), \( f \) produces generally uncalibrated scores \( s_y(X_i) \in [0,1] \) \citep{calibrated-scores}. The predicted label is $\hat{Y}_i(X_i) = \arg\max_{y \in \mathcal{Y}} s_y(X_i),$ so that $s(X_i) = \max_{y \in \mathcal{Y}} s_y(X_i).$ \\
At regular intervals \( T \), we collect a batch \( \mathcal{D} = \{(X_i, Y_i)\}_{i=1}^N \) of instances drawn independently from a potentially time-varying distribution \( \mathcal{F} \). Given a random draw \( (X,Y) \sim \mathcal{F} \), we define the target performance metric $\epsilon := \mathbb{E}_{\mathcal{F}}[Z]$ \citep{model-scores}, where \( Z \) is a generic error signal, e.g. the misclassification rate:
\begin{align} \label{eq:epsilon}
\epsilon = \mathbb{E}_{\mathcal{F}}[Z] = \mathbb{E}_{\mathcal{F}}[\mathds{1}_{(\hat{Y} \neq Y)}],
\end{align}

where $\mathds{1}_{(\hat{Y} \neq Y)}$ is the indicator function that equals \( 1 \) if $(\hat{Y} \neq Y)$ holds, \( 0 \) otherwise.
For notation simplicity, we omit explicit dependence on $T$ and $\mathcal{D}$.
% Extension to the multiclass setting is trivial by adjusting $Z$ accordingly.

Given the dataset $\mathcal{D}$, we could estimate $\epsilon$ by
$\hat{\epsilon}_{\mathcal{D}} = \frac{1}{N}\sum_{i=1}^N \mathds{1}_{(\hat{Y}_i \neq Y_i)}$. 
% $\hat{\epsilon}_{\mathcal{D}} = \frac{1}{N}\sum_{i=1}^N \mathds{1}_{(\hat{Y}_i \neq Y_i)}\mathds{1}_{(Y_i = 1)}$. 
However, while we have access to $X$ and to the outputs of $f$ for all $1 \leq i \leq N, Y$ is unknown. Given a finite labeling budget, we assume access to an oracle that reveals the true label $Y$ only when queried, allowing us to observe labels for a subset $\mathcal{D}^{\text{sample}} \subset \mathcal{D}$ of size $n \ll N$.
We will select these instances according to a sampling design $\pi$, which is a probability distribution over all subsets of size $n$ in $\mathcal{D}$. We denote by $\pi_i>0$ the likelihood that the $i$-th instance is included in $\mathcal{D}^{\text{sample}}$. Using the labelled data, we then obtain an estimator $\hat{\epsilon}$ of $\epsilon$:
\[
\mathcal{D}^{\text{sample}} = \left\{ (X^{\text{sample}}_i, Y^{\text{sample}}_i) \right\}_{i=1}^{n} \Rightarrow \hat{\epsilon} = g_{\pi}(\mathcal{D}^{\text{sample}}),
\]
with $g_{\pi}: \mathcal{D}^{\text{sample}} \rightarrow \mathbb{R}^{+}$ depending on $\pi$.

We measure the efficiency of the estimator $\hat{\epsilon}$ of $\epsilon$ in terms of its Mean Squared Error \citep{casella2024statistical}:
\begin{align*}
    \operatorname{MSE}_{\pi}(\hat{\epsilon}, \epsilon)&=\mathbb{E}_{\mathcal{F}}\left[\mathbb{E}_\pi\left[(\hat{\epsilon}-\epsilon)^2\right]\right] \\
&\approx\left(\mathbb{E}_{\mathcal{F}}\left[\mathbb{E}_\pi[\hat{\epsilon}]\right]-\epsilon \right)^2+\mathbb{E}_{\mathcal{F}}\left[\operatorname{Var}_\pi(\hat{\epsilon})\right],
\end{align*}
thanks to the bias-variance decomposition and $\operatorname{Var}_{\mathcal{F}}\left(\mathbb{E}_\pi[\hat{\epsilon}]\right)$ being small when $n \ll N$. Since we will only look at design-unbiased estimators, the MSE will simplify to the variance:
%\vspace{0.2cm}
\begin{align} \label{eq:mse}
    \operatorname{MSE}_{\pi}(\hat{\epsilon}, \epsilon) \approx \mathbb{E}_{\mathcal{F}}\left[\operatorname{Var}_\pi(\hat{\epsilon}) \right].
\end{align}
\begin{definition}[Relative Efficiency of two Estimators] \label{def:relative_efficiency}
We define the efficiency \citep{fogliato} of estimator $\hat{\epsilon}^{(1)}$ relative to $\widehat{\epsilon}^{(2)}$ under a sampling design $\pi$ as the inverse of the ratio of their MSEs, i.e. $\operatorname{MSE}_\pi\left(\hat{\epsilon}^{(2)},\epsilon\right) / \operatorname{MSE}_\pi\left(\hat{\epsilon}^{(1)}, \epsilon\right)$.
\end{definition}
In the following sections, we say that estimator $\hat{\epsilon}^{(1)}$ is more efficient than $\hat{\epsilon}^{(2)}$ when the relative efficiency is greater than one.

% Methods
\section{METHODS} \label{sec:methods}

Stratified importance sampling (SIS) is an established variance-reduction tool in Monte Carlo analysis. We introduce here a novel general framework that readapts SIS for monitoring classification models after deployment, targeting efficient estimation of the error signal $\epsilon=\mathbb{E}_{\mathcal{F}}[Z]$ under finite labeling budgets. SIS combines (i) stratification to control between–stratum variability and (ii) importance weighting within strata to concentrate labels on informative instances. Algorithm \ref{alg:sis} indicates how to implement SIS in the monitoring of deployed models. 
We assume \emph{proportional allocation}, $n_j/n \to w_j := |\mathcal{D}_j|/|\mathcal{D}|$, which is standard in the deployment monitoring setting \citep{lohr2021sampling}, since stratum-specific variances are unknown without labeled data. While Neyman allocation is theoretically optimal when variances are known, it is generally infeasible in practice for model monitoring, as estimating per-stratum variances would itself require substantial labeled data \citep{wesolowski2023recursive}.

\begin{algorithm}[t]
\caption{Stratified Importance Sampling (SIS) for Defect Rate Monitoring}
\label{alg:sis}
\begin{algorithmic}[1]
\Require Dataset $\mathcal{D}=\{(X_i,\hat{Y_i})\}_{i=1}^N$, model scores $s(x_i) \ \forall i$, budget $n\ll N$
\Ensure Efficient estimator $\displaystyle\hat\epsilon_{\mathrm{SIS}}$
\State Partition $\mathcal{D}$ into approximate strata $\{\mathcal{D}_j\}_{j=1}^P$ \citep{druck2011toward} \label{alg-line:partition}
\State $N_j\gets|\mathcal{D}_j|$, \ $w_j\gets N_j/N$, \ $n_j\gets\lfloor n\,w_j\rfloor$
\State Define SIS importance distribution $q(x)$ as function of $s(x_i)$, $i\in {1, \ldots, N}$ \citep{poms2021low} \label{alg-line:is}
\For{$j=1,\dots,P$}
  \State Sample $\mathcal{D}_{j}^{\text{sample}}\subset\mathcal{D}_j$ of size $n_j$ w.p.\ $q_j(\cdot) := q(\cdot | \cdot \in D_j)$
  \State Query the oracle to retrieve labels $Y^{\text{sample}}_j$ for $X^{\text{sample}}_j\in\mathcal{D}_{j}^{\text{sample}}$
  \State Set $Z_j\gets\mathds{1}_{(\hat Y^{\text{sample}}_j\neq Y^{\text{sample}}_j)}$ and compute $\displaystyle\hat\epsilon_{\mathrm{IS},j}$ according to Equation \eqref{eq:sis-definition}
\EndFor
\State Estimate $\displaystyle\hat\epsilon_{\mathrm{SIS}}$ through Equation \eqref{eq:sis-definition}
\end{algorithmic}
\end{algorithm}

Let the dataset $\mathcal{D}=\{x_i\}_{i=1}^N$ be partitioned into $P$ disjoint strata $\{\mathcal{D}_j\}_{j=1}^P$, with stratum weights $w_j = |\mathcal{D}_j|/|\mathcal{D}|$ representing the finite-population proportions. Within each stratum $j$, we consider as baseline distribution $p_j$ the uniform distribution over $\mathcal{D}_j$, which corresponds to simple random sampling restricted to that stratum \citep{sawade2010active}. Let $q(x)$ denote a global proposal distribution over $\mathcal{D}$, typically induced by model scores, and let $q_j(\cdot)=q(\cdot \mid S=j)$ be its conditional restriction to stratum $j$.\\
The SIS estimator is defined as:
\begin{align}
&\hat{\epsilon}_{\mathrm{SIS}} = \sum_{j=1}^P w_j \,\hat{\epsilon}_{\mathrm{IS},j}, \text{ s.t. }
\label{eq:sis-definition}
\end{align}
\begin{align*} \hat{\epsilon}_{\mathrm{IS},j} = \frac{1}{n_j} \sum_{k=1}^{n_j}
\, z\!\left(X^{\mathrm{sample}}_{j,k}\right)\,
\frac{p_j\!\left(X^{\mathrm{sample}}_{j,k}\right)}{q_j\!\left(X^{\mathrm{sample}}_{j,k}\right)},
\end{align*}
% where $z(x)=\mathds{1}_{(\hat{Y}(x)\neq 
% Y(x))}\mathds{1}_{(Y(x)=1)}$
where $z(x)=\mathds{1}_{(\hat{Y}(x)\neq 
Y(x))}$, and the sample in stratum $j$ consists of $n_j$ independent draws from $q_j$.

Under the conditions of
(i) absolute continuity $p_j \ll q_j$ for all $j$;
(ii) finite second moments $\mathbb{E}_{q_j}\!\big[(z(X)\,p_j(X)/q_j(X))^2\big]<\infty$;
(iii) independent sampling within strata and proportional allocation $n_j/n\to w_j$,
the estimator $\hat{\epsilon}_{\mathrm{SIS}}$ is design-unbiased and consistent (Appendix~\ref{sec:appendix-theoretical-properties-sis}). Additionally, the theoretical results in Theorem \ref{thm:SIS_IS} and \ref{thm:SIS_SRS} establish that SIS achieves lower mean squared error (MSE) than both standard importance sampling (IS) and stratified random sampling (SRS) under mild conditions. These finite-sample statistical guarantees establish SIS as a more efficient efficient estimator compared to baselines, according to definition \ref{def:relative_efficiency}.

\vspace{0.3cm}
\begin{restatable}[Efficiency improvement of the Stratified Importance Sampling estimator against Importance Sampling]{theorem}{SisIs}\label{thm:SIS_IS}
Consider the task of estimating $\epsilon=\mathbb{E}[Z]$, where $Z=z(X)\in\{0,1\}$. Let the population be partitioned into $P$ strata with population weights $w_j=\mathbb{P}_p(S=j)$, and let $r_j=\mathbb{P}_q(S=j)$ denote the stratum marginals under an Importance Sampling proposal $q$.
Within each stratum $j$, write $q_j(\cdot)=q(\cdot\mid S=j)$ and $p_j(\cdot)=p(\cdot\mid S=j)$.
Define
\[
R \;=\; z(X)\,\frac{p(X)}{q(X)}, \ \pi_j \;=\; \mathbb{E}_p[z(X)\mid S=j],
\]
and the within–stratum IS second–moment gap
\[
T_j \;:=\; \operatorname{Var}_q(R\mid S=j)\;=\; \mathbb{E}_{q_j}[R^2]-\pi_j^2.
\]

Assume:
(i) absolute continuity within strata $p_j\ll q_j$ for all $j$;
(ii) finite second moment $\mathbb{E}_q[R^2]<\infty$; 
(iii) proportional allocation for SIS, $n_j/n\to w_j$.
Consider $\widehat\epsilon_{\mathrm{SIS}}$ as defined in Equation \eqref{eq:sis-definition} and Importance sampling estimator defined as $\widehat\epsilon_{\mathrm{IS}}:=\frac1n\sum_{i=1}^n R_i$. 

Then, it holds that $\operatorname{MSE}(\widehat\epsilon_{\mathrm{SIS}},\epsilon) \le \operatorname{MSE}(\widehat\epsilon_{\mathrm{IS}},\epsilon) \text{ iff } \sum_{j=1}^P (r_j-w_j)\,T_j \;+\; \operatorname{Var}_{S\sim r}(\pi_S)\;\ge 0.$

\end{restatable}

Theorem \ref{thm:SIS_IS} states how the variance gap between SIS and IS splits into two components: 
\begin{align*}
    &\operatorname{Var}(\widehat\epsilon_{\mathrm{IS}}) - \operatorname{Var}(\widehat\epsilon_{\mathrm{SIS}})
    \;= \notag \\
    & = \;\frac{1}{n}\left\{\sum_{j=1}^P (r_j-w_j)\,T_j \;+\; \operatorname{Var}_{S\sim r}(\pi_S)\right\}.
    \label{eq:var-reduction-sis-is}
\end{align*}
The \emph{importance sampling mismatch} term $\sum_{j=1}^P (r_j-w_j)T_j$ captures how the proposal $q$ allocates sampling probability across strata relative to proportional allocation; and the second \emph{inter-stratum variance} term $\operatorname{Var}_{S\sim r}(\pi_S)$, captures the isolation of variance between strata. Thus, SIS yields a more efficient estimator when the variance reduction achieved by eliminating inter-stratum variability outweighs the variance reduction of the importance sampling mismatch.

% -------UPDATED: SIS VS SRIS AISTATS------------
\vspace{0.3cm}
\begin{restatable}[Efficiency improvement of the Stratified Importance Sampling estimator against Stratified Random Sampling]{theorem}{SisSrs} \label{thm:SIS_SRS}
Let $\hat{\epsilon}_{\text{SIS}}$ denote the estimator defined in Equation~\eqref{eq:sis-definition}, and let $\hat{\epsilon}_{\text{SRS}} := \sum_{j=1}^P w_j \hat{\epsilon}_j$ denote the stratified random sampling estimator with proportional allocation, where $\hat{\epsilon}_j = \tfrac{1}{n_j}\sum_{k=1}^{n_j} z(X^{\text{sample}}_{j,k})$ is the stratum mean estimator.

Define the within-stratum variance gap as the difference between  
(i) the variance of the importance-weighted estimator when instances in stratum $j$ are sampled from the proposal distribution $q(x)$,  
and  
(ii) the variance of the simple random sampling estimator when instances are sampled uniformly from the original distribution $p(x)$:
\[
\Delta_j(q_j) 
:= \operatorname{Var}_q\!\Big(z(X)\tfrac{p(X)}{q(X)} \,\big|\, S=j\Big)
- \operatorname{Var}_p\!\big(z(X)\mid S=j\big).
\]

Then, under proportional allocation, $\operatorname{MSE}(\widehat\epsilon_{\mathrm{SIS}},\epsilon)\le \operatorname{MSE}(\widehat\epsilon_{\mathrm{SRS}},\epsilon)$ iff $\sum_{j=1}^P w_j\,\Delta_j(q_j)<0$.

\end{restatable}

Theorem \ref{thm:SIS_SRS} indicates how SIS need not reduce variance in every stratum; it suffices that reductions in some strata dominate increases elsewhere, i.e., the weighted average gap is negative. This explains robustness when using a noisy proposal distribution per stratum: strong strata compensate for weak ones. Appendix~\ref{sec:appendix-proofs} proves the theorem to hold under the condition that the weighted average of within-stratum variance gaps is negative and indicates how, in the absence of such condition, SIS does not guarantee MSE reduction with respect to SRS.

These two theoretical statements establish that SIS achieves a more label-efficient estimator relative to both IS and SRS under finite samples and mild assumptions. 
Since true labels are unknown in deployment, IS and SRS each provide only limited variance reduction when relying solely on non–calibrated scores or surrogate stratification variables. 
By combining stratification and importance weighting, SIS inherits robustness to uncalibrated scores from SRS while preserving the efficiency gains of IS when scores are informative. 
Although SIS cannot guarantee variance reduction when proposal mismatch is too severe (Theorem~\ref{thm:SIS_IS}) or when the weighted average gap is non–negative (Theorem~\ref{thm:SIS_SRS}), it consistently leverages approximate strata and imperfect proposals to deliver efficiency improvements beyond those attainable by either method alone.

% Results
\section{EXPERIMENTS}
\label{sec:results}

\begin{table*}[t]
\centering
\caption{Summary of datasets and classifiers used in experiments. Defect rate refers to the proportion of misclassified instances under the deployed model.}
\label{tab:datasets}
\begin{tabular}{l l l l l r r r}
\toprule
Dataset & Data Type & \#\ Classes & Classifier & Test set size & Accuracy (\%) & Defect rate (\%) \\
\midrule
Breast Cancer (BCW) & Tabular & Binary & LogReg & 285 & 97.54 & 2.46 \\
Digits  & Tabular & 10 & LogReg & 899 & 96.22 & 3.78 \\
Credit default & Tabular & Binary & LogReg & 30,000 & 81.12 & 18.88 \\
MNIST  & Image & 10 & CNN & 10,000 & 99.01 & 0.99 \\
CIFAR10  & Image & 10 & CNN & 10,000 & 81.17 & 18.83 \\
Proprietary   & Tabular & Binary & XGBoost & 342,035 & 99.56 & 0.44 \\
\bottomrule
\end{tabular}
\end{table*}

\paragraph{Experimental design.} 
We evaluate SIS across six datasets spanning both tabular and image classification tasks, binary and multiclass settings, and regimes with widely different defect rates (Table~\ref{tab:datasets}). This diversity ensures that our conclusions are not tied to a single data modality or model family, but instead reflect a range of different model monitoring settings encountered in practice.
For each dataset, we simulate sample-based model monitoring by acquiring labels according to five designs: Random Sampling (RS), Stratified Random Sampling (SRS), Importance Sampling (IS), FILA, an adaptive variant of SRS by \cite{guan2022fila}, Adaptive Importance Sampling (AIS) by \cite{marchant2021needle}), and our proposed implementation of Stratified Importance Sampling (SIS), which implements the same stratification and proposal and SRS and IS, respectively. 

\begin{table*}[t]
\captionsetup{skip=5pt}
\centering
\caption{Relative efficiency (RE) with respect to Random Sampling ($\uparrow$ better). 
Representatives configurations are shown, Appendix D includes the full grid of different results and details the strata and proposal distribution implementations. 
Values are mean $\pm$ standard error over repeated 250 - 1,000 simulations.  
Highlighted the design with highest RE, per dataset and sample size.
}
\label{tab:detailed-efficiency}
\small
\begin{tabular}{l l l l l l}
\toprule
Dataset & Sample size & Strata & Proposal & Design & RE vs. Random \\
\midrule
\textbf{MNIST} 
 & 100  & PRED & -- & Stratified & 0.95 $\pm$ 0.07 \\
 & 100  & PRED & $q(x) \propto \text{score}^{0.25}$ & Importance & 7.84 $\pm$ 0.55 \\
 & 100  & PRED & -- & Adaptive stratified & 1.02 $\pm$ 0.08 \\
 & 100  & PRED & $q(x) \propto \text{score}^{0.25}$ & \textbf{Adaptive IS} & \textbf{10.31} $\pm$ \textbf{0.86} \\
 & 100  & PRED & $q(x) \propto \text{score}^{0.25}$ & SIS & 8.97 $\pm$ 0.70 \\
 & 1400 & PRED & -- & Stratified & 0.99 $\pm$ 0.06 \\
 & 1400 & PRED & $q(x) \propto \text{score}^{0.25}$ & Importance & 7.82 $\pm$ 0.48 \\
 & 1400 & PRED & -- & Adaptive stratified & 1.07 $\pm$ 0.07 \\
 & 1400 & PRED & $q(x) \propto \text{score}^{0.25}$ & \textbf{Adaptive IS} & \textbf{12.82} $\pm$ \textbf{0.82} \\
 & 1400 & PRED & $q(x) \propto \text{score}^{0.25}$ & SIS & 9.27 $\pm$ 0.58 \\
\midrule
\textbf{BCW} 
 & 50  &\texttt{worst\_perimeter}-4x3 & -- & Stratified & 1.24 $\pm$ 0.09 \\
 & 50  & -- & $q(x) \propto \text{score}^{0.8}$ & Importance & 0.92 $\pm$ 0.07 \\
  & 50  & \texttt{worst\_perimeter}-4x3 & -- & Adaptive stratified & 1.48 $\pm$ 0.10 \\
 & 50  & -- & $q(x) \propto \text{score}^{0.8}$ & Adaptive IS & 1.22 $\pm$ 0.08 \\
 & 50  & \texttt{worst\_perimeter}-4x3 & $q(x) \propto \text{score}^{0.8}$ & \textbf{SIS} & \textbf{1.65} $\pm$ \textbf{0.11} \\
\midrule
\textbf{CIFAR10} 
 & 200 & BR-BIN & -- & Stratified & 1.01 $\pm$ 0.11 \\
  & 200 & -- & $q(x) \propto \text{score}^{0.50}$ & Importance & 1.12 $\pm$ 0.31 \\
 & 200 & BR-BIN & $q(x) \propto \text{score}^{0.50}$ & \textbf{SIS} & \textbf{1.66} $\pm$ \textbf{0.18} \\
  & 750 & BR-BIN & -- & Stratified & 1.21 $\pm$ 0.13 \\
  & 750 & -- & $q(x) \propto \text{score}^{0.50}$ & \textbf{Importance} & \textbf{1.62} $\pm$ \textbf{0.21} \\
 & 750 & BR-BIN & $q(x) \propto \text{score}^{0.50}$ & SIS & 1.59 $\pm$0.19 \\
\midrule
\textbf{Digits} 
 % & 40 & PCA2-4x3 & -- & Stratified & 1.36 $\pm$ 0.17 \\
  & 40 & -- & $q(x) \propto \text{score}^{1.0}$ & Importance & 2.28 $\pm$ 0.63 \\
 & 40 & PCA2-4x3 & $q(x) \propto \text{score}^{1.0}$ & \textbf{SIS} & \textbf{2.52} $\pm$ \textbf{0.52} \\
 % & 80 & PCA2-3x3 & -- & Stratified & 1.36 $\pm$ 0.17 \\
  & 80 & -- & $q(x) \propto \text{score}^{0.5}$ & \textbf{Importance} & \textbf{3.44} $\pm$ \textbf{0.48} \\
 & 80 & PCA2-3x3 & $q(x) \propto \text{score}^{0.5}$ & SIS & 2.20 $\pm$ 0.29 \\
\midrule
\textbf{Credit default} 
 & 500 & X2 & -- & Stratified & 1.11 $\pm$ 0.14 \\
  & 500 & -- & $q(x) \propto \text{score}^{0.5}$ & Importance & 1.02 $\pm$ 0.13 \\
 & 500 & X2 & $q(x) \propto \text{score}^{0.5}$ & \textbf{SIS} & \textbf{1.42} $\pm$ \textbf{0.18} \\
\midrule
\textbf{Proprietary} 
 & 500 & PRED x ModelVersion & -- & Stratified & 2.57 $\pm$ 0.17 \\
 & 500 & -- & $q(x) \propto \text{score}^{0.5}$ & Importance & 4.18 $\pm$ 1.04 \\
 & 500 & PRED x ModelVersion & $q(x) \propto \text{score}^{0.5}$ & \textbf{SIS} & \textbf{6.32} $\pm$ \textbf{0.46} \\
 & 2500 & PRED x ModelVersion & -- & Stratified & 2.43 $\pm$ 0.15 \\
 & 2500 & -- & $q(x) \propto \text{score}^{0.5}$ & \textbf{Importance} & \textbf{4.86} $\pm$ \textbf{0.53} \\
 & 2500 & PRED x ModelVersion & $q(x) \propto \text{score}^{0.5}$ & SIS & 4.53 $\pm$ 0.92 \\
\bottomrule
\end{tabular}
\end{table*}

\begin{figure}[t] % or [!t] / [H] if using float package
  \centering
  \includegraphics[width=\linewidth]{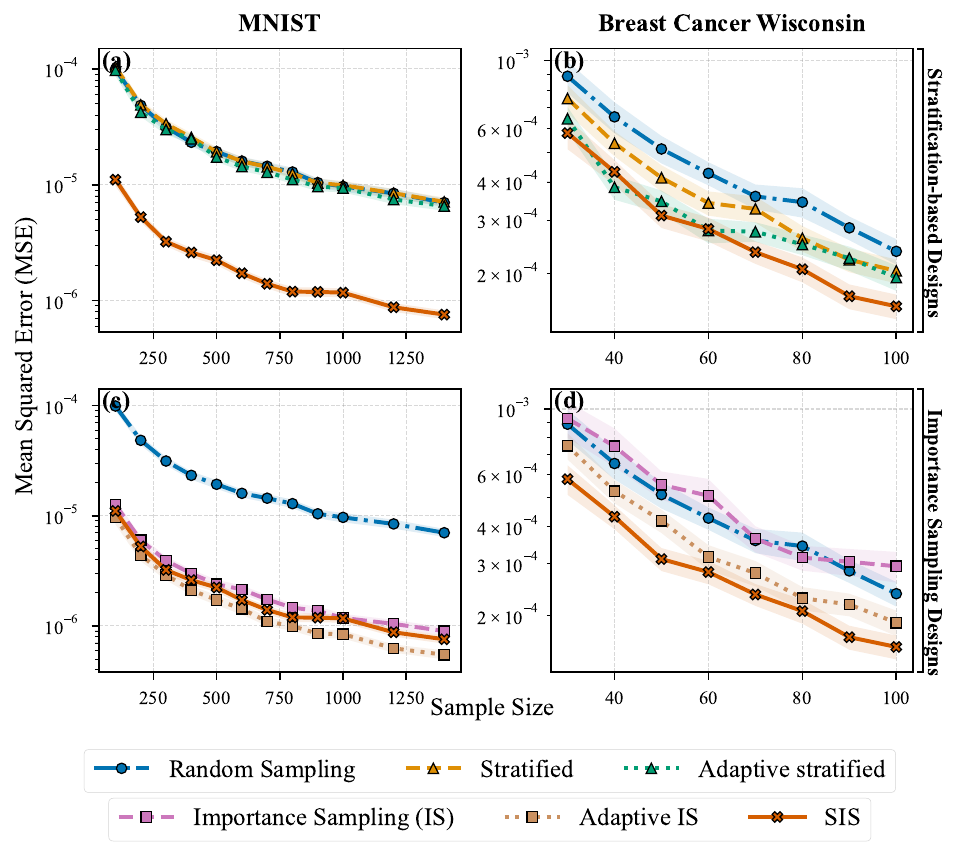}
  \caption{%
  \textbf{Comparison of sampling designs on MNIST and BCW.}
  Lines show mean MSE (log scale) versus sample size; shaded regions indicate
  \iftrue % set to \iffalse if you used ±1 SE
  95\% confidence bands ($\pm 1.96\,\mathrm{SE}$)
  \else
  $\pm 1\,\mathrm{SE}$
  \fi
  across simulations.
  }
  \label{fig:sampling-designs}
\end{figure}

\paragraph{Implementation of strata and proposal distributions.}
In our experiments, strata were constructed from simple partitions of every dataset: e.g., quantile bins of continuous covariates (perimeter in BCW, PCA components in Digits, age/credit limit in Credit Default), predicted class and brightness bins for CIFAR10, and cross-bins combining feature and score information. To ensure stability, very small strata were merged into the nearest neighbor by median score. Proposal distributions were defined globally as $q(x) \propto s(x)^\alpha$, where $s(x)$ is a score proxy such as classifier uncertainty, entropy, or confidence, and $\alpha$ tunes the sharpness of allocation. For SIS, the proposal is restricted and renormalized within each stratum. Full implementation details, including the precise feature choices, binning strategies, and search grids for $\alpha$, are provided in Appendix~D.

\paragraph{Efficiency gains from proxy strata and noisy proposal distributions.} 
Table~\ref{tab:detailed-efficiency} indicates the relative efficiencies with respect to random sampling across representative strata and proposal distribution configurations,
spanning the six different datasets and sample sizes implemented. Appendix D details the full experimental grid.
Consistent with the theoretical analysis in Theorems~\ref{thm:SIS_IS}–\ref{thm:SIS_SRS}, SIS generally improves efficiency compared to both SRS and IS across diverse datasets. Notably, SIS exhibits robustness in scenarios with ill-defined proposal distributions or weak stratification. For instance, in the BCW dataset with $n=50$, SIS achieves $\text{RE}=1.65\pm0.11$, correcting the negative impact of the poorly aligned IS proposal distribution, which underperforms random sampling with a \text{RE} of $0.92\pm0.07$. Similarly, in the Credit Default dataset at $n=500$, SIS reaches $\text{RE}=1.42\pm0.18$, improving upon both SRS ($1.11\pm0.14$) and IS ($1.02\pm0.13$). Even when stratification is only weakly informative, as in MNIST, where Stratified Sampling empirically does not improve random sampling, SIS nonetheless leverages these proxy strata to outperform efficiency of IS, attaining $\text{RE}=8.97\pm0.70$ at $n=100$. These results highlight SIS’s ability to exploit partial structure in scores and strata, yielding measurable variance reductions even in noisy or mis-specified settings.

% 2. Detail where it underperforms, analysis also of sample size
At the same time, SIS does not guarantee improvements in all regimes. When the sufficient conditions of Theorems~\ref{thm:SIS_IS}–\ref{thm:SIS_SRS} are not satisfied—such as under severe proposal mismatch or uninformative stratification, SIS may fail to outperform either IS or SRS. For example, in the Digits dataset ($n=80$), IS attains higher relative efficiency ($3.44\pm0.48$) than SIS ($2.20\pm0.29$), and in CIFAR10 with $n=750$, IS ($1.62\pm0.21$) slightly surpasses SIS ($1.59\pm0.19$). Interestingly, across several datasets (e.g., Digits, CIFAR10), SIS tends to be more competitive at smaller sample sizes, suggesting that variance stabilization from stratification can be particularly beneficial under tight labeling regimes. This behavior is consistent with our theoretical characterization: while SIS cannot eliminate the weaknesses of IS or SRS individually, it reliably moderates their variance extremes.

\paragraph{Comparison to adaptive approaches.}
Table~\ref{tab:detailed-efficiency} also benchmarks SIS with adaptive designs. While SIS provides a strong non-adaptive baseline, adaptive updates can surpass it in certain regimes by exploiting small pilot labels to refine allocations. Figure \ref{fig:sampling-designs} shows this comparison for the MNIST and BCW datasets. For instance, on \textsc{MNIST}, \emph{Adaptive IS} consistently attains lower MSE than SIS across budgets: at $n=100$, it achieves $9.62{\times}10^{-6}$ versus $1.11{\times}10^{-5}$ for SIS ; at $n=500$, $1.72{\times}10^{-6}$ versus $2.23{\times}10^{-6}$ and at $n=1400$, $5.47{\times}10^{-7}$ versus $7.56{\times}10^{-7}$. By contrast, on Breast Cancer, adaptivity yields only a narrow advantage: at $n=40$, \emph{Adaptive Stratified} achieves $3.84{\times}10^{-4}$ versus $4.32{\times}10^{-4}$ for SIS, Besides this sample size allocation, SIS regains the lowest estimation error. These results suggest that while adaptive methods can deliver sharper MSE reductions, they may also induce instability or diminishing returns outside of specific high enough label budgets, where the available sample size does not enable for extensive sample allocation iterations.

Overall, the empirical results corroborate the theoretical analysis of section \ref{sec:methods}: SIS consistently reduces mean squared error relative to random sampling and typically improves upon SRS and IS across diverse datasets and budgets (Table~\ref{tab:detailed-efficiency}, Fig.~\ref{fig:sampling-designs}). The gains are particularly strong when proposals are misspecified or stratification is only weakly informative, such as on BCW at $n{=}50$, where SIS ($\text{RE}=1.65\pm0.11$) corrects the negative impact of IS ($0.92\pm0.07$), or on Credit Default at $n{=}500$, where SIS ($1.42\pm0.18$) outperforms both SRS and IS. Even in settings where stratification adds little signal, as in MNIST, SIS still leverages proxy structure to achieve strong efficiency ($\text{RE}=8.97\pm0.70$ at $n{=}100$). At the same time, SIS is not uniformly dominant: under severe proposal mismatch or negligible inter-stratum variance, IS or SRS can be competitive or superior (e.g., Digits at $n{=}80$, CIFAR10 at $n{=}750$, and the Proprietary dataset at $n{=}2500$). Interestingly, SIS often remains more competitive at smaller budgets. Comparisons with adaptive designs further reveal that while SIS offers a robust, operationally lightweight baseline for one-shot batch monitoring, adaptive IS or stratified variants can surpass it when small pilot labels are available and proposals can be stably updated; for instance, on MNIST, where the gap in efficiency widens at higher budgets. Taken together, these findings establish SIS as a practical and principled framework for label-efficient monitoring: it inherits the stability of stratified designs under noisy proxies, retains the efficiency of IS when scores are informative, and provides a principled non-adaptive baseline that can be complemented by adaptive methods when operationally feasible.

% Related work
\section{RELATED WORK} \label{sec:related-work}

Before concluding, we position our work in relation to an extended prior research on multiple monitoring approaches and SIS applications, involving sampling, adaptive or label-free methods.

\textbf {Sampling in model monitoring}.
Stratified random sampling (SRS) has long been used to improve evaluation efficiency, yielding notable reductions in sample size and mean squared error (MSE) compared to random sampling \citep{horvitz1952generalization, fogliato, online-stratified}. Applications include domains such as speech recognition \citep{speech-stratified-sampling}, but they often rely on partial labels to define strata \citep{fogliato, liberty2016stratified}. Importance sampling (IS) has similarly improved efficiency by focusing on uncertain predictions \citep{importance-sampling-paper}, though typically under the assumption of calibrated model scores. Despite these advances, the joint use of stratification and importance weighting remains largely unstudied for model monitoring, and no general framework has been established in machine learning.

\textbf{Stratified Importance Sampling applications}.
Prior work has explored Stratified Importance Sampling (SIS) in domain-specific tasks, but rarely generalized. Glasserman et al. \citep{glasserman1999importance} combined stratification with IS for Value-at-Risk estimation, assuming Gaussian risk factors and delta approximations. Xiao et al. \citep{xiao2020reliability} proposed AK-SIS for structural reliability, using Kriging surrogates to design strata and densities, achieving efficiency but tied to surrogate decision and not a general analysis. Kim et al. \citep{kim2023stratified} stratified default patterns in credit portfolios and optimized per-stratum IS distributions, proving variance reduction guarantees against IS but not comparing with SRS. Beyond rare-event estimation, Berjisian et al. \citep{berjisian2019developing} used SIS for choice-set generation in pedestrian destination modeling.
By contrast, our framework establishes SIS as a principled approach for post-deployment classification monitoring across multiple metrics, requiring only mild assumptions, no surrogates or parametric models, and characterizing when efficiency gains are guaranteed even with noisy proxies, sub-optimal stratification and one-shot batch labels.

\textbf{Adaptive sampling approaches}.
Adaptive IS methods have been developed for highly imbalanced settings, often with a focus on F-scores. Marchant et al. \citep{marchant2021needle} introduced a batch-adaptive sampler with asymptotic consistency and CLT guarantees, while Marchant et al. \citep{marchant2017search} proposed OASIS, a sequential adaptive sampler for entity resolution with Bayesian updates achieving optimal asymptotic variance. Guan et al. \citep{guan2022fila} designed FILA, an adaptive stratified auditor for accuracy with asymptotically optimal allocation. These methods rely on parametric updates and asymptotic analysis, which can be operationally heavy and fragile under misspecified scores. Our work instead provides a non-adaptive SIS framework suited to one-shot batch annotation in deployed classifiers, proving finite-sample MSE improvements without assumptions on score calibration or stratum optimality, while remaining extensible to adaptive designs if operationally feasible. Still, such extension would require additional structural assumptions on how the importance distribution is updated and would likely sacrifice some of the generality that our static design achieves.

\textbf{Beyond sampling-based Monitoring.} Label-free or label-sparse monitoring has been studied through unsupervised shift detection and proxy-based estimation. Approaches include two-sample tests and representation-based divergences \citep{rabanser2019failing, roschewitz2024automatic}, drift-aware pipelines for CTR prediction \citep{liu2023always}, sequential shift detection using proxy errors \citep{i2024sequential}, and accuracy estimation under covariate shift via calibrated scores and density ratios \citep{bialek2024estimating}. These methods surface distributional changes or alarms but do not provide direct, statistically efficient estimates of error rates under label budgets. Our contribution instead centers on SIS as a direct sampling and estimation framework for accurate monitoring of classification metrics in practical scenarios.

% Conclusion
\section{CONCLUSION} \label{sec:conclusion}

Label-efficient model monitoring remains a central challenge, particularly under strict labeling budgets, uncalibrated scores, one-shot batch annotation, and extremely low defect rates. We present a framework for applying stratified importance sampling (SIS) to this setting, showing how the joint use of importance weighting and stratification reduces mean squared error (MSE) relative to either method alone. Unlike importance sampling (IS) or stratified random sampling (SRS), which are each limited by proxy noise in scores or imperfect stratification attributes, SIS combines their complementary strengths: IS prioritizes high-impact regions, while SRS stabilizes variance across subpopulations. Our theoretical results establish that, under mild conditions, this framework yields unbiased estimators and strict finite-sample MSE improvements over both IS and SRS. While SIS does not eliminate the intrinsic weaknesses of IS or SRS in extreme cases, it consistently delivers label-efficiency gains in realistic monitoring regimes.

Extensive experiments across binary and multiclass tasks, spanning both tabular and image datasets, confirm these findings. SIS consistently outperformed IS and SRS, reducing estimator variance and label requirements across diverse data modalities and defect rates. The experiments also highlight the robustness of SIS: when scores are informative it matches IS-level efficiency while avoiding IS instability under poor proposals, and when strata align with error structure it achieves substantial variance reduction beyond SRS. Comparisons with adaptive designs further underscored SIS as a strong non-adaptive baseline: while adaptive IS could achieve lower MSE when the labeling budget enables the refinement with pilot labels, SIS remained competitive across budgets and data types without incurring the operational overhead of adaptive updates.

\textbf{Limitations and future work.}
Future work includes extending SIS to broader performance metrics, adaptive monitoring setups, and dynamic weighting under non-stationary data. A promising direction is to develop adaptive stratified importance sampling, where proposals are updated across batches or over time. While our analysis shows SIS improves efficiency under mild assumptions, it also reveals limitations: variance reduction is not guaranteed when proposal mismatch is severe (Theorem \ref{thm:SIS_IS}) or when the weighted average of within-stratum variance gaps is non-negative (Theorem \ref{thm:SIS_SRS}). Adaptive variants may mitigate these cases, but they would require stronger assumptions on how to update strata or importance distributions, potentially sacrificing the generality and the operational lightness of the static framework.

\bibliography{references}

\clearpage
\appendix
\thispagestyle{empty}

\onecolumn

\section*{SUPPLEMENTARY MATERIAL}

\section{BASELINE SAMPLING DESIGNS}
\label{sec:baseline-sampling-designs}

In this Appendix section we formalize the limitations of existing sampling designs for model monitoring from prior work.

\textbf{Random Sampling.} Random Sampling (RS) yields an unbiased estimator \(\hat{\epsilon}\) of the defect rate \(\epsilon\), but suffers from high variance when defects are rare \citep{cochran1977sampling}, limiting its efficiency under constrained annotation budgets.

\textbf{Stratified Random Sampling.} Stratified Random Sampling (SRS) aims to reduce estimator variance by partitioning the sample space \( \mathcal{D} \) into disjoint strata \( \{\mathcal{D}_1, \dots, \mathcal{D}_P\} \) such that \( \mathcal{D} = \mathcal{D}_1 \cup \cdots \cup \mathcal{D}_P \) and \( \mathcal{D}_s \cap \mathcal{D}_t = \emptyset \) for all \( s \neq t \). The goal is to minimize the variance of the error estimator \( \operatorname{Var}_\pi(\hat{\epsilon}) \) by ensuring low variability within each stratum. The optimal stratification minimizes the sum of squared deviations around the stratum means:
\[
\{\mathcal{D}_1^{*}, \dots, \mathcal{D}_P^{*}\} = \operatorname*{argmin}_{\{\mathcal{D}_1, \dots, \mathcal{D}_P\}} \sum_{j=1}^P \sum_{k=1}^{N_j} \left(\mathds{1}_{(\hat{Y}_{j,k} \neq Y_{j,k})} - \epsilon_j \right)^2,
\]
where \( \epsilon_j \) is the unknown defect rate within stratum \( \mathcal{D}_j \), which has size \( N_j \).
The effectiveness of SRS critically depends on the quality of the stratification. In practice, achieving optimal stratification is infeasible, as true labels are unavailable at sampling time. Consequently, stratification procedures are often sub-optimal, particularly when auxiliary stratification variables are only weakly correlated with the true defect probability \citep{chen2020biascorrected, liberty2016stratified}, or when strata are poorly aligned with model error rates \citep{fogliato}. Moreover, the benefits of SRS vary widely across domains: substantial variance reduction has been observed on structured benchmarks such as CIFAR-10 \citep{pmlb-benchmark}, whereas gains are often marginal in noisy or highly imbalanced real-world datasets.
Recent work in literature proposed partially labeled datasets to guide stratification \citep{kossen2022active}.

\textbf{Importance Sampling.} Importance Sampling (IS) reallocates sampling probability toward instances likely to exhibit model misclassification, aiming to reduce the MSE of the estimator in Equation \eqref{eq:mse}. Formally, IS samples according to a importance distribution \(q(x)\) and corrects the estimator via weights \(w(x) = p(x)/q(x)\), where \(p(x)\) denotes the original data distribution. For estimating $\epsilon$ when \(Z = z(x)\) is binary, the optimal importance distribution \(q^*(x) \propto z(x)p(x)\) minimizes the MSE, concentrating sampling on defective instances. The more aggressively \(q(x)\) concentrates on defective cases, the larger the MSE reduction: \(q(x) \geq p(x)\) for all \(x\) such that \(z(x) = 1\).  In practice, \(q^*(x)\) is unattainable without access to labels prior to sampling. Practitioners rely on proxy scores \(s(x)\) to approximate defect likelihood, but miscalibration or noise in these scores can yield unstable and inefficient estimators, particularly when weights \(w(x)\) become extreme \citep{owen2013monte, cornebise2008adaptive}.

Critically, neither SRS nor IS uniformly dominates the other: SRS is robust when stratification captures meaningful structure \citep{fisch2024stratified}, while IS is efficient for low defect rates and reliable scores ~\citep{marchant2021needle}. Motivated by these complementary strengths and weaknesses, SIS design combines both principles: stratification provides robustness against uncalibrated scores, while importance weighting within strata concentrates sampling on the most informative instances, maximizing efficiency under real-world operational constraints.
\newpage
\section{PROOFS OF THE EFFICIENCY IMPROVEMENT OF SIS}
\label{sec:appendix-proofs}

In this section we provide proofs and derivations for the main properties of the Stratified Importance Sampling (SIS) estimator in Equation \eqref{eq:sis-definition}. We start by recalling and proving the two key theorems on SIS presented in section \ref{sec:methods}.

\SisIs*
\begin{proof}[\textbf{Proof}]
Let $S\in\{1,\dots,P\}$ denote the stratum index. Under the baseline data distribution $p$, the stratum weights are $w_j=\mathbb{P}_p(S=j)$ with conditional distribution per stratum $p_j(\cdot)=p(\cdot\mid S=j)$. Under the global IS proposal $q$, let $r_j=\mathbb{P}_q(S=j)$ and $q_j(\cdot)=q(\cdot\mid S=j)$.
Define
\[
R=z(X)\frac{p(X)}{q(X)},\qquad 
\pi_j=\mathbb{E}_q[R\mid S=j]=\mathbb{E}_{p_j}[z(X)],\qquad
T_j=\operatorname{Var}_q(R\mid S=j).
\]

\medskip
\noindent\emph{Unbiasedness of IS and SIS.}  
By conditioning on strata,
\[
\mathbb{E}_q[R]
=\sum_{j=1}^P r_j \pi_j
=\sum_{j=1}^P w_j \pi_j
=\mathbb{E}_p[z(X)]
=\epsilon.
\]
Thus $\widehat\epsilon_{\mathrm{IS}}$  is unbiased for $\epsilon$. We prove the unbiasedness of $\widehat\epsilon_{\mathrm{SIS}}$ in Appendix \ref{sec:appendix-theoretical-properties-sis}.

\medskip
\noindent\emph{Variance of IS.}  
Since $\widehat\epsilon_{\mathrm{IS}}=\tfrac1n\sum_{i=1}^n R_i$ with i.i.d.\ $R_i\sim q$,
\[
\operatorname{Var}(\widehat\epsilon_{\mathrm{IS}})
=\frac{1}{n}\operatorname{Var}_q(R).
\]
By the law of total variance w.r.t.\ the stratum index $S\sim r$, 
\[
\operatorname{Var}_q(R)
= \sum_{j=1}^P r_j\,\operatorname{Var}_q(R\mid S=j)
  + \operatorname{Var}_{S\sim r}(\pi_S),
\]
that is, as the average within--stratum variance (under proposal marginals $r_j$) plus the variance of the conditional means across strata.

Hence
\[
\operatorname{Var}(\widehat\epsilon_{\mathrm{IS}})
=\frac{1}{n}\Bigg\{\sum_{j=1}^P r_j T_j + \operatorname{Var}_{S\sim r}(\pi_S)\Bigg\}.
\]

\medskip
\noindent\emph{Variance of SIS.}  
Within each stratum $j$, SIS computes
\[
\widehat\epsilon_{\mathrm{IS},j}
=\frac{1}{n_j}\sum_{k=1}^{n_j} R_{j,k},\qquad R_{j,k}\stackrel{\text{i.i.d.}}{\sim}q_j,
\]
so that $\operatorname{Var}(\widehat\epsilon_{\mathrm{IS},j})=T_j/n_j$. Since strata are independent,
\[
\operatorname{Var}(\widehat\epsilon_{\mathrm{SIS}})
=\sum_{j=1}^P w_j^2\,\operatorname{Var}(\widehat\epsilon_{\mathrm{IS},j})
=\sum_{j=1}^P \frac{w_j^2}{n_j}\,T_j.
\]
Under proportional allocation $n_j/n\to w_j$, we have
$\tfrac{w_j^2}{n_j}=\tfrac{1}{n}\,w_j+o(1/n)$, and thus
\[
\operatorname{Var}(\widehat\epsilon_{\mathrm{SIS}})
=\frac{1}{n}\sum_{j=1}^P w_j T_j\;+\;o\!\left(\tfrac{1}{n}\right).
\]
If $n_j=nw_j$ exactly, the $o(1/n)$ remainder vanishes.

\medskip
\noindent\emph{Variance gap.}  
Subtracting,
\[
\operatorname{Var}(\widehat\epsilon_{\mathrm{IS}})-\operatorname{Var}(\widehat\epsilon_{\mathrm{SIS}})
=\frac{1}{n}\Bigg\{\sum_{j=1}^P (r_j-w_j)\,T_j + \operatorname{Var}_{S\sim r}(\pi_S)\Bigg\}
+o\!\left(\tfrac{1}{n}\right).
\]
With exact proportional allocation $n_j=nw_j$, the $o(1/n)$ term disappears, yielding the claimed identity.

\medskip
\noindent\emph{MSE comparison.}  
Since both estimators are unbiased, their MSEs equal their variances. Therefore,
\[
\operatorname{MSE}(\widehat\epsilon_{\mathrm{SIS}},\epsilon)\;\le\;
\operatorname{MSE}(\widehat\epsilon_{\mathrm{IS}},\epsilon)
\quad\Longleftrightarrow\quad
\sum_{j=1}^P (r_j-w_j)\,T_j + \operatorname{Var}_{S\sim r}(\pi_S)\;\ge 0.
\]

\end{proof}

\SisSrs*
\begin{proof}[\textbf{Proof}]
Let $S\in\{1,\dots,P\}$ be the stratum index, with stratum weights $w_j=\mathbb{P}_p(S=j)$ and conditionals $p_j(\cdot)=p(\cdot\mid S=j)$. Let $q$ be a proposal with stratum marginals $r_j=\mathbb{P}_q(S=j)$ and conditionals $q_j(\cdot)=q(\cdot\mid S=j)$. Define
\[
R=z(X)\frac{p(X)}{q(X)},\qquad 
T_j=\operatorname{Var}_q(R\mid S=j),\qquad 
V_j=\operatorname{Var}_p\big(z(X)\mid S=j\big).
\]
We assume $p_j\ll q_j$ and $\mathbb{E}_q[R^2]<\infty$. Both estimators target $\epsilon=\mathbb{E}_p[z(X)]$ and are unbiased:
\[
\mathbb{E}[\hat\epsilon_{\mathrm{SIS}}]=\epsilon,
\qquad
\mathbb{E}[\hat\epsilon_{\mathrm{SRS}}]=\sum_{j=1}^P w_j\,\mathbb{E}_{p_j}[z(X)]=\epsilon.
\]

\noindent\emph{Variance of SIS.}
Within stratum $j$, SIS forms the IS average
$\widehat\epsilon_{\mathrm{IS},j}=\frac{1}{n_j}\sum_{k=1}^{n_j} R_{j,k}$ with $R_{j,k}\stackrel{\text{i.i.d.}}{\sim}q_j$,
so $\operatorname{Var}(\widehat\epsilon_{\mathrm{IS},j})=T_j/n_j$.
Since strata are sampled independently,
\[
\operatorname{Var}(\hat\epsilon_{\mathrm{SIS}})
=\sum_{j=1}^P w_j^2\,\operatorname{Var}(\widehat\epsilon_{\mathrm{IS},j})
=\sum_{j=1}^P \frac{w_j^2}{n_j}\,T_j
=\frac{1}{n}\sum_{j=1}^P w_j T_j\;+\;o\!\left(\tfrac{1}{n}\right),
\]
under proportional allocation $n_j/n\to w_j$. If $n_j=nw_j$ exactly, the remainder vanishes and the last equality is exact.

\medskip
\noindent\emph{Variance of SRS.}
Within stratum $j$, SRS draws $X^{\text{sample}}_{j,1},\dots,X^{\text{sample}}_{j,n_j}\stackrel{\text{i.i.d.}}{\sim}p_j$ and computes
$\hat\epsilon_j=\frac{1}{n_j}\sum_{k=1}^{n_j} z(X^{\text{sample}}_{j,k})$,
hence $\operatorname{Var}(\hat\epsilon_j)=V_j/n_j$.
Again using independence across strata,
\[
\operatorname{Var}(\hat\epsilon_{\mathrm{SRS}})
=\sum_{j=1}^P w_j^2\,\operatorname{Var}(\hat\epsilon_j)
=\sum_{j=1}^P \frac{w_j^2}{n_j}\,V_j
=\frac{1}{n}\sum_{j=1}^P w_j V_j\;+\;o\!\left(\tfrac{1}{n}\right),
\]
with exact equality if $n_j=nw_j$.

\medskip
\noindent\emph{Variance gap.}
Subtracting the two displays and using the definition
\[
\Delta_j(q_j)\;:=\;\operatorname{Var}_q\!\Big(z(X)\tfrac{p(X)}{q(X)}\,\big|\,S=j\Big)\;-\;\operatorname{Var}_p\!\big(z(X)\mid S=j\big)\;=\;T_j - V_j,
\]
we obtain
\[
\operatorname{Var}(\hat\epsilon_{\mathrm{SIS}})-\operatorname{Var}(\hat\epsilon_{\mathrm{SRS}})
=\frac{1}{n}\sum_{j=1}^P w_j\,\Delta_j(q_j)\;+\;o\!\left(\tfrac{1}{n}\right).
\]
Under exact proportional allocation $n_j=nw_j$, the $o(1/n)$ term disappears and the stated identity holds.

\medskip
\noindent\emph{MSE comparison.}
Since both estimators are unbiased for $\epsilon$, their MSEs equal their variances. Hence
\[
\operatorname{MSE}(\widehat\epsilon_{\mathrm{SIS}},\epsilon)\;\le\;\operatorname{MSE}(\widehat\epsilon_{\mathrm{SRS}},\epsilon)
\quad\Longleftrightarrow\quad
\sum_{j=1}^P w_j\,\Delta_j(q_j)\;\le\;0,
\]
with equality iff $\sum_{j=1}^P w_j\,\Delta_j(q_j)=0$ (again, exactly under $n_j=nw_j$ and up to $o(1/n)$ otherwise).
\end{proof}
\newpage
\section{THEORETICAL PROPERTIES OF SIS ESTIMATOR}
\label{sec:appendix-theoretical-properties-sis}
In the following proofs, we show the unbiasedness and consistency properties of the proposed Stratified Importance Sampling (SIS) estimator in Equation \eqref{eq:sis-definition}.

\begin{prop}
    The Stratified Importance Sampling (SIS) estimator \(\hat{\epsilon}_{\text{SIS}}\), defined in Equation \ref{eq:sis-definition}, is unbiased. That is, its expected value equals the unknown performance metric \(\epsilon\):
    \[
    \mathbb{E}[\hat{\epsilon}_{\text{SIS}}] = \epsilon.
    \]
\end{prop}

\begin{proof}
    The SIS estimator is defined as:
    \[
    \hat{\epsilon}_{\text{SIS}} = \sum_{j=1}^P w_j \hat{\epsilon}_{\text{IS}, j}, \text { where }     \hat{\epsilon}_{\text{IS}, j} = \frac{1}{n_j} \sum_{k=1}^{n_j} z(X^{\text{sample}}_{j,k}) \frac{p_j(X^{\text{sample}}_{j,k})}{q_j(X^{\text{sample}}_{j,k})},
    \]
    where \(\hat{\epsilon}_{\text{IS}, j}\) is the stratum-specific importance sampling (IS) estimator. This estimator \(\hat{\epsilon}_{\text{IS}, j}\) is unbiased, achieving:
    \[
    \mathbb{E}[\hat{\epsilon}_{\text{IS}, j}] = \epsilon_j,
    \]
    where \(\epsilon_j\) is the performance metric within stratum \(j\).
    
    Using the linearity of expectation we get:
    \[
    \mathbb{E}[\hat{\epsilon}_{\text{SIS}}] = \mathbb{E} \left[ \sum_{j=1}^P w_j \hat{\epsilon}_{\text{IS}, j} \right]
    = \sum_{j=1}^P w_j \ \mathbb{E}[\hat{\epsilon}_{\text{IS}, j}] = \sum_{j=1}^P w_j \epsilon_j.
    \]
    By the definition of stratified sampling, the overall defect rate \(\epsilon\) is a weighted sum of the within-stratum defect rates, and this concludes the proof:
    \[
    \epsilon = \sum_{j=1}^P w_j \ \epsilon_j \Rightarrow  \mathbb{E}[\hat{\epsilon}_{\text{SIS}}] = \epsilon. 
    \]
\end{proof}

\begin{prop}
    The Stratified Importance Sampling (SIS) estimator \(\hat{\epsilon}_{\text{SIS}}\), defined in Equation \eqref{eq:sis-definition}, is consistent. That is, as the sample size \(n\) grows to infinity, the estimator converges in probability to the unknown defect rate \(\epsilon\):
    \[
    \hat{\epsilon}_{\text{SIS}} \xrightarrow{\mathbb{P}} \epsilon \quad \text{as } n \to \infty.
    \]
\end{prop}

\begin{proof}
     The SIS estimator is defined as:
    \[
    \hat{\epsilon}_{\text{SIS}} = \sum_{j=1}^P w_j \hat{\epsilon}_{\text{IS}, j}, \text { where }     \hat{\epsilon}_{\text{IS}, j} = \frac{1}{n_j} \sum_{k=1}^{n_j} z(X^{\text{sample}}_{j,k}) \frac{p_j(X^{\text{sample}}_{j,k})}{q_j(X^{\text{sample}}_{i,j})},
    \]
    since $\hat{\epsilon}_{\text{IS}, j}$ reflects the importance sampling (IS) design within each stratum $j$.
    From the Law of Large Numbers (LLN), the importance sampling estimator \(\hat{\epsilon}_{\text{IS}, j}\) is consistent for \(\epsilon_j\). Specifically:
    \[
    \hat{\epsilon}_{\text{IS}, j} \xrightarrow{\mathbb{P}} \epsilon_j \quad \text{as } n_j \to \infty \quad \forall j \in \{1, \ldots, P\},
    \]
    where \(n_j\) is the number of samples in stratum \(j\).
    
    Using the consistency of \(\hat{\epsilon}_{\text{IS}, j}\), it follows that:
    \[
    \mathbb{E}[\hat{\epsilon}_{\text{IS}, j}] \to \epsilon_j \quad \text{and} \quad \text{Var}(\hat{\epsilon}_{\text{IS}, j}) \to 0 \quad \text{as } n_j \to \infty \quad \forall j \in \{1, \ldots, P\}.
    \]
    Since \(w_j\) is a fixed proportion for each stratum, and the weighted sum of consistent estimators is itself consistent, we have:
    \[
    \hat{\epsilon}_{\text{SRIS}} = \sum_{j=1}^P w_j \cdot \hat{\epsilon}_{\text{IS}, j} \xrightarrow{\mathbb{P}} \sum_{j=1}^P w_j \cdot \epsilon_j.
    \]
    
    By the definition of stratified sampling, the parameter \(\epsilon\) is given by:
    \[
    \epsilon = \sum_{j=1}^P w_j \ \epsilon_j \Rightarrow \hat{\epsilon}_{\text{SIS}} \xrightarrow{\mathbb{P}} \epsilon,
    \]
    which proves the consistency of SIS estimator.
\end{proof}
\newpage
\section{Extended Experimental Results}
\label{sec:appendix-extended-results}

This appendix provides details supporting the empirical evaluation in Section~4. We document the datasets used, the exact implementation of strata and proposal distributions for the Stratified Importance Sampling (SIS) framework, and the complete grid of experimental results. The goal is to support reproducibility and clarify every experimental configuration.

\subsection{Description of the Datasets}
\label{subsec:appendix-extended-results-descriptions}

We report additional characteristics of the six datasets summarized in Table~1 of the main text.\footnote{Five of the datasets are publicly available, we make available the Proprietary dataset as part of this submission.}

\paragraph{Breast Cancer Wisconsin (BCW)\footnote{\url{https://archive.ics.uci.edu/dataset/17/breast+cancer+wisconsin+diagnostic}}}
Binary tabular dataset from the UCI repository containing 569 instances and 30 standardized features describing cell nuclei measurements. We train a logistic regression classifier using a 50/50 train–test split and use the test set (285 samples) for monitoring. The model achieves 97.54\% accuracy, corresponding to a 2.46\% defect rate.

\paragraph{Digits: Pen-Based Recognition of Handwritten Digits\footnote{\url{https://archive.ics.uci.edu/dataset/81/pen+based+recognition+of+handwritten+digits}}}
A multiclass (10-class) handwritten digit classification dataset from the UCI repository with 1{,}792 instances. A multinomial logistic regression classifier is trained on a 50/50 split, reaching 96.22\% accuracy (3.78\% defect rate).

\paragraph{Credit Default\footnote{\url{https://www.openml.org/search?type=data&exact_name=default-of-credit-card-clients&id=42477}}}
A tabular dataset from the OpenML repository containing 30{,}000 credit card clients, used to predict default probability (binary classification). We train a logistic regression model, obtaining 81.12\% accuracy (18.88\% defect rate).

\paragraph{MNIST\footnote{\url{https://docs.pytorch.org/vision/main/generated/torchvision.datasets.MNIST.html}}}
The canonical 10-class image benchmark (60{,}000 training, 10{,}000 test samples). We train a small convolutional neural network (CNN) achieving 99.01\% test accuracy (0.99\% defect rate). This dataset evaluates performance in the extremely low-error regime, where random sampling is highly inefficient.

\paragraph{CIFAR10\footnote{\url{https://docs.pytorch.org/vision/main/generated/torchvision.datasets.CIFAR10.html}}}
A 10-class color image dataset with 50{,}000 training and 10{,}000 test instances. A compact CNN reaches 81.17\% test accuracy (18.83\% defect rate). The higher variability and multimodal structure make it useful for assessing robustness of SIS under noisy or weak proxies.

\paragraph{Proprietary Dataset.}
A large-scale binary classification dataset from an anonymized industrial deployment (342{,}035 instances). An XGBoost model serves as the deployed classifier, achieving 99.56\% accuracy (0.44\% defect rate). This dataset reflects a realistic post-deployment monitoring scenario with near-perfect model accuracy and heterogeneous subpopulations.

\subsection{Detailed Implementation of Strata and Proposal Distributions}
\label{subsec:appendix-extended-results-strata-proposals}

This section describes the construction of strata and proposal distributions used across all experiments.  
The same high-level principles apply to all datasets:

\begin{enumerate}[]
  \item Strata are built by quantile binning either interpretable features or model-based quantities 
  (e.g., predicted class, confidence, brightness), ensuring at least two strata and merging those representing less than $0.5\%$ of samples.
  \item Proposals $q(x)$ follow the generic form $q(x) \propto s(x)^{\alpha}$, 
  with $s(x)$ an uncertainty-related score and $\alpha \in [0.25, 1.0]$.
  \item The same global $q$ is used for both IS and SIS; for SIS, $q$ is renormalized within each stratum $j$, 
  i.e.\ $q_j(x) = q(x) / \sum_{x' \in S_j} q(x')$.
\end{enumerate}

\paragraph{Breast Cancer Wisconsin (BCW).}
We stratify the test pool by combining quantile bins on the  feature \texttt{worst\_perimeter} ($4$ bins) and the classifier score (\texttt{score}, $3$ bins). Within each feature bin, score quantiles are computed, yielding up to $12$ candidate strata. Strata with fewer than $3$ samples are merged with the closest by median score.  
The proposal is $q(x) \propto \texttt{score}^{\alpha}$ with $\alpha = 0.8$.

\paragraph{Digits.}
Strata are defined by 2D binning on the second principal component (\texttt{pca2}) and classifier confidence (\texttt{conf}), with $(n_\text{feat\_bins}, n_\text{score\_bins}) = (3,3)$. Small strata ($<5$ samples) are merged based on confidence medians.  
The proposal uses classifier uncertainty, $s(x) = 1 - \texttt{conf}(x)$, with $q(x) \propto s(x)^{\alpha}$ and $\alpha = 1.0$.

\paragraph{Credit Default.}
We consider categorical and continuous features as potential stratification variables, keeping those with at least two non-tiny categories.  
Proposals are built from out-of-fold probabilities $\hat{p}$ of default:  
\[
q(x) \propto s(x)^{\alpha}, \quad s(x) \in \{0.5 - |\hat{p}-0.5|, \; -[\hat{p}\log\hat{p}+(1-\hat{p})\log(1-\hat{p})], \; \hat{p}, \; 1-\hat{p}\}.
\]
We test $\alpha = 0.5$ and  $\alpha = 1.0$.

\paragraph{MNIST.}
Each test sample is assigned to strata formed by either the predicted class (\texttt{PRED}), brightness quintiles (\texttt{BR\_BIN}), or their Cartesian product (\texttt{PREDxBR}) when all cells exceed $0.5\%$ of the data.  
The proposal is $q(x) \propto (1 - p_{\max}(x))^{\alpha}$ with $p_{\max}$ the classifier’s maximum softmax probability. We use $\alpha = 0.5$.

\paragraph{CIFAR10.}
Stratification follows the same design as MNIST: categorical strata on predicted class (\texttt{PRED}), brightness quintiles (\texttt{BR\_BIN}), or their combination (\texttt{PREDxBR}).  
Proposal families coincide with MNIST ($s(x) \in \{1-p_{\max}, \text{entropy}, p_{\max}\}$) and $\alpha = 0.5$.

\paragraph{Proprietary Dataset.}
We form categorical strata from deployment metadata: model version and predicted class (\texttt{model\_version}, \texttt{PRED}), and their interaction.  
Quantile bins on the model score (\texttt{ml\_score})—5 and 10 bins—are added for hybrid strata such as \texttt{version×score\_q5}.  
Proposal families include: 
\[
s(x) \in \left\{0.5 - |p-0.5|,\; -[p\log p+(1-p)\log(1-p)],\; p,\; 1-p\right\},
\]
with $\alpha=0.5$.

\begin{table}[h]
\centering
\caption{Summary of strata construction and proposal distributions across datasets.}
\begin{tabular}{lcccc}
\toprule
Dataset & Strata variables & Proposal family $s(x)$ & $\alpha$ \\
\midrule
BCW & \texttt{worst\_perimeter}, score & $\texttt{score}$ & 0.8 \\
Digits & PCA(2), confidence & $1-\texttt{conf}$ & 1.0 \\
Credit Default & Categorical + binned numeric & $\{0.5-|p-0.5|, H(p), p, 1-p\}$ & 0.25–1.0 \\
MNIST & predicted class × brightness & $1-p_{\max}$ & 0.5 \\
CIFAR10 & predicted class × brightness  & $1-p_{\max}$, entropy, $p_{\max}$ & 0.5 \\
Proprietary & version × predicted class / score bins  & $0.5-|p-0.5|$, entropy, $p$, $1-p$ & 0.5 \\
\bottomrule
\end{tabular}
\end{table}

\subsection{Detailed Results Grid}
\label{subsec:appendix-extended-results-results-grid}
\subsubsection{Breast Cancer Wisconsin}

Table~\ref{tab:bcw_detailed_results_scaled} reports the detailed grid of Monte Carlo results for the Breast Cancer Wisconsin (BCW) dataset. 
The reported values correspond to the mean squared error (MSE), scaled by $10^{-3}$, and the relative efficiency (RE) with respect to Random Sampling (RS). 
Across all sampling budgets, Stratified Importance Sampling (SIS) consistently outperforms both IS and SRS, achieving between $1.5\times$ and $1.7\times$ efficiency gains. 
Adaptive stratified variants (FILA) also perform robustly, with slightly higher variability but comparable improvements at moderate sample sizes ($n \in [40, 80]$). 
In contrast, pure importance sampling (IS) often underperforms random sampling, confirming that stratification is crucial when the proposal is only weakly aligned with the underlying defect structure. 
Overall, the BCW results demonstrate the stability of SIS under small-sample conditions and its clear variance reduction relative to both non-stratified and adaptive baselines.

\begin{table*}[t]
\centering
\small
\begin{tabular}{rllllll}
\toprule
\makecell{Sample \\ size} & Sample design & Strata (if any) & Proposal (if any) & MSE ($\times 10^{-3}$) & \makecell{RE w.r.t. \\ Random} \\
\midrule
30 & Random Sampling & — & — & 0.89 $\pm$ 0.05 & 1.00 $\pm$ 0.09 \\
30 & Stratified & worst\_perimeter, bins=4x3 & — & 0.75 $\pm$ 0.04 & 1.18 $\pm$ 0.10 \\
30 & Importance & — & $q(x) \propto \texttt{score}^{\alpha}$, $\alpha=0.8$ & 0.93 $\pm$ 0.07 & 0.96 $\pm$ 0.09 \\
30 & SIS & worst\_perimeter, bins=4x3 & $q(x) \propto \texttt{score}^{\alpha}$, $\alpha=0.8$ & 0.58 $\pm$ 0.03 & 1.54 $\pm$ 0.13 \\
30 & Adaptive stratified & — & — & 0.65 $\pm$ 0.03 & 1.38 $\pm$ 0.11 \\
30 & Adaptive IS & — & — & 0.80 $\pm$ 0.04 & 1.18 $\pm$ 0.09 \\
40 & Random Sampling & — & — & 0.74 $\pm$ 0.04 & 1.00 $\pm$ 0.08 \\
40 & Stratified & worst\_perimeter, bins=4x3 & — & 0.55 $\pm$ 0.03 & 1.22 $\pm$ 0.08 \\
40 & Importance & — & $q(x) \propto \texttt{score}^{\alpha}$, $\alpha=0.8$ & 0.69 $\pm$ 0.05 & 0.87 $\pm$ 0.08 \\
40 & SIS & worst\_perimeter, bins=4x3 & $q(x) \propto \texttt{score}^{\alpha}$, $\alpha=0.8$ & 0.39 $\pm$ 0.03 & 1.51 $\pm$ 0.12 \\
40 & Adaptive stratified & — & — & 0.43 $\pm$ 0.03 & 1.70 $\pm$ 0.13 \\
40 & Adaptive IS & — & — & 0.52 $\pm$ 0.04 & 1.24 $\pm$ 0.09 \\
50 & Random Sampling & — & — & 0.53 $\pm$ 0.04 & 1.00 $\pm$ 0.07 \\
50 & Stratified & worst\_perimeter, bins=4x3 & — & 0.40 $\pm$ 0.03 & 1.24 $\pm$ 0.09 \\
50 & Importance & — & $q(x) \propto \texttt{score}^{\alpha}$, $\alpha=0.8$ & 0.57 $\pm$ 0.04 & 0.92 $\pm$ 0.07 \\
50 & SIS & worst\_perimeter, bins=4x3 & $q(x) \propto \texttt{score}^{\alpha}$, $\alpha=0.8$ & 0.29 $\pm$ 0.03 & 1.65 $\pm$ 0.11 \\
50 & Adaptive stratified & — & — & 0.32 $\pm$ 0.03 & 1.48 $\pm$ 0.10 \\
50 & Adaptive IS & — & — & 0.41 $\pm$ 0.03 & 1.22 $\pm$ 0.08 \\
60 & Random Sampling & — & — & 0.43 $\pm$ 0.03 & 1.00 $\pm$ 0.06 \\
60 & Stratified & worst\_perimeter, bins=4x3 & — & 0.33 $\pm$ 0.03 & 1.25 $\pm$ 0.08 \\
60 & Importance & — & $q(x) \propto \texttt{score}^{\alpha}$, $\alpha=0.8$ & 0.46 $\pm$ 0.03 & 0.84 $\pm$ 0.07 \\
60 & SIS & worst\_perimeter, bins=4x3 & $q(x) \propto \texttt{score}^{\alpha}$, $\alpha=0.8$ & 0.28 $\pm$ 0.03 & 1.52 $\pm$ 0.10 \\
60 & Adaptive stratified & — & — & 0.29 $\pm$ 0.03 & 1.54 $\pm$ 0.10 \\
60 & Adaptive IS & — & — & 0.34 $\pm$ 0.03 & 1.35 $\pm$ 0.09 \\
70 & Random Sampling & — & — & 0.36 $\pm$ 0.03 & 1.00 $\pm$ 0.06 \\
70 & Stratified & worst\_perimeter, bins=4x3 & — & 0.33 $\pm$ 0.03 & 1.10 $\pm$ 0.08 \\
70 & Importance & — & $q(x) \propto \texttt{score}^{\alpha}$, $\alpha=0.8$ & 0.36 $\pm$ 0.03 & 0.99 $\pm$ 0.07 \\
70 & SIS & worst\_perimeter, bins=4x3 & $q(x) \propto \texttt{score}^{\alpha}$, $\alpha=0.8$ & 0.28 $\pm$ 0.03 & 1.52 $\pm$ 0.10 \\
70 & Adaptive stratified & — & — & 0.28 $\pm$ 0.03 & 1.31 $\pm$ 0.08 \\
70 & Adaptive IS & — & — & 0.32 $\pm$ 0.03 & 1.29 $\pm$ 0.09 \\
80 & Random Sampling & — & — & 0.38 $\pm$ 0.04 & 1.00 $\pm$ 0.07 \\
80 & Stratified & worst\_perimeter, bins=4x3 & — & 0.29 $\pm$ 0.03 & 1.31 $\pm$ 0.10 \\
80 & Importance & — & $q(x) \propto \texttt{score}^{\alpha}$, $\alpha=0.8$ & 0.41 $\pm$ 0.03 & 1.09 $\pm$ 0.08 \\
80 & SIS & worst\_perimeter, bins=4x3 & $q(x) \propto \texttt{score}^{\alpha}$, $\alpha=0.8$ & 0.23 $\pm$ 0.03 & 1.66 $\pm$ 0.12 \\
80 & Adaptive stratified & — & — & 0.27 $\pm$ 0.03 & 1.38 $\pm$ 0.10 \\
80 & Adaptive IS & — & — & 0.29 $\pm$ 0.03 & 1.50 $\pm$ 0.11 \\
90 & Random Sampling & — & — & 0.34 $\pm$ 0.03 & 1.00 $\pm$ 0.06 \\
90 & Stratified & worst\_perimeter, bins=4x3 & — & 0.27 $\pm$ 0.03 & 1.27 $\pm$ 0.08 \\
90 & Importance & — & $q(x) \propto \texttt{score}^{\alpha}$, $\alpha=0.8$ & 0.36 $\pm$ 0.03 & 0.93 $\pm$ 0.07 \\
90 & SIS & worst\_perimeter, bins=4x3 & $q(x) \propto \texttt{score}^{\alpha}$, $\alpha=0.8$ & 0.20 $\pm$ 0.03 & 1.68 $\pm$ 0.11 \\
90 & Adaptive stratified & — & — & 0.23 $\pm$ 0.03 & 1.26 $\pm$ 0.08 \\
90 & Adaptive IS & — & — & 0.27 $\pm$ 0.03 & 1.30 $\pm$ 0.08 \\
100 & Random Sampling & — & — & 0.33 $\pm$ 0.03 & 1.00 $\pm$ 0.06 \\
100 & Stratified & worst\_perimeter, bins=4x3 & — & 0.28 $\pm$ 0.03 & 1.16 $\pm$ 0.08 \\
100 & Importance & — & $q(x) \propto \texttt{score}^{\alpha}$, $\alpha=0.8$ & 0.34 $\pm$ 0.03 & 0.81 $\pm$ 0.06 \\
100 & SIS & worst\_perimeter, bins=4x3 & $q(x) \propto \texttt{score}^{\alpha}$, $\alpha=0.8$ & 0.22 $\pm$ 0.03 & 1.51 $\pm$ 0.10 \\
100 & Adaptive stratified & — & — & 0.27 $\pm$ 0.03 & 1.22 $\pm$ 0.08 \\
100 & Adaptive IS & — & — & 0.29 $\pm$ 0.03 & 1.25 $\pm$ 0.09 \\

\bottomrule
\end{tabular}
\caption{Detailed mean squared error (MSE, scaled by $10^{-3}$) and relative efficiency (RE) w.r.t.~Random Sampling for each sampling design on Breast Cancer Wisconsin (BCW). Strata: \texttt{worst\_perimeter} with 4 feature quantiles $\times$ 3 score quantiles; Proposal for IS/SIS: $q(x) \propto \texttt{score}^{\alpha}$ with $\alpha=0.8$. Values are mean $\pm$ SE over 1,000 Monte-Carlo simulations.}
\label{tab:bcw_detailed_results_scaled}
\end{table*}

\subsubsection{Digits}
Table \ref{tab:digits_detailed_results_scaled} reports the detailed Monte Carlo results for the Digits dataset. 
Consistent with the findings on BCW, Stratified Importance Sampling (SIS) improves estimation efficiency relative to both Random Sampling (RS) and Stratified Random Sampling (SRS), reaching relative efficiencies between $2.2\times$ and $2.5\times$ for small budgets ($n \le 60$). 
When the uncertainty-based proposal $q(x) \propto (1 - \mathrm{conf}(x))^{1.0}$ is well aligned with classifier uncertainty, both IS and SIS outperform RS substantially, with SIS achieving comparable or slightly higher stability across repetitions. 
At $n = 80$, Importance Sampling (IS) achieves the highest single efficiency ($\mathrm{RE}=1.65$).

\begin{table*}[t]
\centering
\small
\begin{tabular}{rllllll}
\toprule
\makecell{Sample \\ size} & Sample design & Strata (if any) & Proposal (if any) & MSE ($\times 10^{-3}$) & \makecell{RE w.r.t. \\ Random} \\
\midrule
40 & Random Sampling & — & — & 0.92 $\pm$ 0.08 & 1.00 $\pm$ 0.00 \\
40 & Stratified & pca2, bins=4x3 & — & 0.82 $\pm$ 0.09 & 1.30 $\pm$ 0.18 \\
40 & Importance & — & $q(x)\propto (1-\texttt{conf})^{\alpha}$, $\alpha=1.0$ & 0.39 $\pm$ 0.09 & 2.28 $\pm$ 0.63 \\
40 & SIS & pca2, bins=4x3 & $q(x)\propto (1-\texttt{conf})^{\alpha}$, $\alpha=1.0$ & 0.41 $\pm$ 0.07 & 2.52 $\pm$ 0.52 \\
\addlinespace
60 & Random Sampling & — & — & 0.61 $\pm$ 0.06 & 1.00 $\pm$ 0.00 \\
60 & Stratified & pca2, bins=4x3 & — & 0.52 $\pm$ 0.06 & 1.22 $\pm$ 0.14 \\
60 & Importance & — & $q(x)\propto (1-\texttt{conf})^{\alpha}$, $\alpha=1.0$ & 0.28 $\pm$ 0.05 & 2.20 $\pm$ 0.26 \\
60 & SIS & pca2, bins=4x3 & $q(x)\propto (1-\texttt{conf})^{\alpha}$, $\alpha=1.0$ & 0.29 $\pm$ 0.05 & 2.35 $\pm$ 0.28 \\
\addlinespace
80 & Random Sampling & — & — & 0.49 $\pm$ 0.05 & 1.00 $\pm$ 0.00 \\
80 & Stratified & pca2, bins=4x3 & — & 0.41 $\pm$ 0.05 & 1.20 $\pm$ 0.12 \\
80 & Importance & — & $q(x)\propto (1-\texttt{conf})^{\alpha}$, $\alpha=1.0$ & 0.33 $\pm$ 0.05 & 1.65 $\pm$ 0.20 \\
80 & SIS & pca2, bins=4x3 & $q(x)\propto (1-\texttt{conf})^{\alpha}$, $\alpha=1.0$ & 0.33 $\pm$ 0.05 & 1.13 $\pm$ 0.13 \\
\bottomrule
\end{tabular}
\caption{Detailed MSE (scaled by $10^{-3}$) and relative efficiency (RE) w.r.t.~Random Sampling for the Digits dataset. Configurations shown use strata on the second principal component (\texttt{pca2}, 3$\times$3 quantile grid) and an uncertainty-based proposal for IS/SIS. Values are mean $\pm$ SE over simulations.}
\label{tab:digits_detailed_results_scaled}
\end{table*}

\subsubsection{Credit Default}

Table \ref{tab:credit_detailed_results_grid} reports the performance of all sampling designs on the Credit~Default dataset. 
In this binary classification task, Stratified Importance Sampling (SIS) consistently improves estimation efficiency relative to both Random Sampling (RS) and Stratified Random Sampling (SRS), achieving relative efficiencies around $1.4\times$ at $n = 500$. 
These gains indicate that even moderately informative categorical or binned numerical strata can provide meaningful variance reduction when combined with an uncertainty-based proposal distribution $q(x) \propto s(x)^{0.5}$. 
Compared to standard Importance Sampling (IS), SIS offers more stable performance.

\begin{table*}[t]
\centering
\small
\begin{tabular}{rllllll}
\toprule
\makecell{Sample \\ size} & Sample design & Strata (if any) & Proposal (if any) & MSE ($\times 10^{-3}$) & \makecell{RE w.r.t. \\ Random} \\
\midrule
200 & Random Sampling & — & — & 0.82 $\pm$ 0.07 & 1.00 $\pm$ 0.00 \\
200 & Stratified & x2 & — & 0.65 $\pm$ 0.06 & 1.25 $\pm$ 0.16 \\
200 & Importance & — &  $q(x)\propto (\texttt{score})^{\alpha}$, ${\alpha}$=0.5 & 0.71 $\pm$ 0.06 & 1.14 $\pm$ 0.14 \\
200 & SIS & x2 &$q(x)\propto (\texttt{score})^{\alpha}$, ${\alpha}$=0.5 & 0.60 $\pm$ 0.05 & 1.35 $\pm$ 0.16 \\
300 & Random Sampling & — & — & 0.50 $\pm$ 0.04 & 1.00 $\pm$ 0.00 \\
300 & Stratified & x2 & — & 0.42 $\pm$ 0.04 & 1.19 $\pm$ 0.15 \\
300 & Importance & — &$q(x)\propto (\texttt{score})^{\alpha}$, ${\alpha}$=0.5 & 0.51 $\pm$ 0.06 & 0.99 $\pm$ 0.14 \\
300 & SIS & x2 &$q(x)\propto (\texttt{score})^{\alpha}$, ${\alpha}$=0.5 & 0.54 $\pm$ 0.05 & 0.93 $\pm$ 0.11 \\
400 & Random Sampling & — & — & 0.37 $\pm$ 0.03 & 1.00 $\pm$ 0.00 \\
400 & Stratified & x2 & — & 0.37 $\pm$ 0.03 & 1.01 $\pm$ 0.12 \\
400 & Importance & — &$q(x)\propto (\texttt{score})^{\alpha}$, ${\alpha}$=0.5 & 0.36 $\pm$ 0.03 & 1.02 $\pm$ 0.13 \\
400 & SIS & x2 &$q(x)\propto (\texttt{score})^{\alpha}$, ${\alpha}$=0.5 & 0.36 $\pm$ 0.03 & 1.04 $\pm$ 0.12 \\
500 & Random Sampling & — & — & 0.34 $\pm$ 0.03 & 1.00 $\pm$ 0.00 \\
500 & Stratified & x2 & — & 0.30 $\pm$ 0.03 & 1.11 $\pm$ 0.14 \\
500 & Importance & — &$q(x)\propto (\texttt{score})^{\alpha}$, ${\alpha}$=0.5 & 0.33 $\pm$ 0.03 & 1.02 $\pm$ 0.13 \\
500 & SIS & x2 &$q(x)\propto (\texttt{score})^{\alpha}$, ${\alpha}$=0.5 & 0.24 $\pm$ 0.02 & 1.42 $\pm$ 0.18 \\
600 & Random Sampling & — & — & 0.25 $\pm$ 0.02 & 1.00 $\pm$ 0.00 \\
600 & Stratified & x2 & — & 0.26 $\pm$ 0.02 & 0.99 $\pm$ 0.12 \\
600 & Importance & — &$q(x)\propto (\texttt{score})^{\alpha}$, ${\alpha}$=0.5 & 0.22 $\pm$ 0.02 & 1.17 $\pm$ 0.14 \\
600 & SIS & x2 &$q(x)\propto (\texttt{score})^{\alpha}$, ${\alpha}$=0.5 & 0.29 $\pm$ 0.02 & 0.87 $\pm$ 0.10 \\
800 & Random Sampling & — & — & 0.20 $\pm$ 0.02 & 1.00 $\pm$ 0.00 \\
800 & Stratified & x2 & — & 0.20 $\pm$ 0.02 & 1.00 $\pm$ 0.13 \\
800 & Importance & — &$q(x)\propto (\texttt{score})^{\alpha}$, ${\alpha}$=0.5 & 0.21 $\pm$ 0.02 & 0.96 $\pm$ 0.12 \\
800 & SIS & x2 &$q(x)\propto (\texttt{score})^{\alpha}$, ${\alpha}$=0.5 & 0.17 $\pm$ 0.01 & 1.18 $\pm$ 0.14 \\
1200 & Random Sampling & — & — & 0.12 $\pm$ 0.01 & 1.00 $\pm$ 0.00 \\
1200 & Stratified & x2 & — & 0.13 $\pm$ 0.01 & 0.92 $\pm$ 0.12 \\
1200 & Importance & — &$q(x)\propto (\texttt{score})^{\alpha}$, ${\alpha}$=0.5 & 0.12 $\pm$ 0.01 & 0.94 $\pm$ 0.12 \\
1200 & SIS & x2 &$q(x)\propto (\texttt{score})^{\alpha}$, ${\alpha}$=0.5 & 0.12 $\pm$ 0.01 & 0.96 $\pm$ 0.12 \\
2000 & Random Sampling & — & — & 0.07 $\pm$ 0.01 & 1.00 $\pm$ 0.00 \\
2000 & Stratified & x2 & — & 0.08 $\pm$ 0.01 & 0.86 $\pm$ 0.11 \\
2000 & Importance & — &$q(x)\propto (\texttt{score})^{\alpha}$, ${\alpha}$=0.5 & 0.09 $\pm$ 0.01 & 0.81 $\pm$ 0.10 \\
2000 & SIS & x2 &$q(x)\propto (\texttt{score})^{\alpha}$, ${\alpha}$=0.5 & 0.07 $\pm$ 0.01 & 0.98 $\pm$ 0.13 \\
\bottomrule
\end{tabular}
\caption{Detailed MSE (scaled by $10^{-3}$) and relative efficiency (RE) w.r.t.~Random Sampling for the Open ML Credit Default dataset. Configurations shown use strata on categorical feature X2, with a proposal built from out-of-fold probabilities for IS and SIS. Values are mean $\pm$ SE over simulations.}
\label{tab:credit_detailed_results_grid}
\end{table*}

\subsubsection{MNIST}
Tables~\ref{tab:mnist_detailed_results_grid_table_100_700} and~\ref{tab:mnist_detailed_results_grid_table_800_1400} present the evaluation results on the MNIST dataset. 
In this low-defect, multiclass image classification setting, Stratified Importance Sampling (SIS) achieves substantial efficiency improvements over Random Sampling (RS) and Stratified Random Sampling (SRS), with relative efficiencies up to approximately $9\times$ at $n = 100$. 
Compared to Importance Sampling (IS), SIS attains similar or slightly lower mean squared error (MSE) while exhibiting greater stability across sample sizes, particularly under tight label budgets. 
When contrasted with adaptive methods, SIS remains competitive at small to moderate budgets: while Adaptive Importance Sampling (AIS) can eventually surpass SIS as the labeling budget increases, the non-adaptive SIS design achieves comparable performance without iterative updates or pilot labels. 

\begin{table*}[t]
\centering
\small
\begin{tabular}{rllllll}
\toprule
\makecell{Sample \\ size} & Sample design & Strata (if any) & Proposal (if any) & MSE ($\times 10^{-6}$) & \makecell{RE w.r.t.\\ Random} \\
\midrule
100 & Random Sampling & PRED & — & 99.20 $\pm$ 4.82 & 1.00 $\pm$ 0.00 \\
100 & Stratified & PRED & — & 104.19 $\pm$ 6.05 & 0.95 $\pm$ 0.07 \\
100 & Importance & PRED &$q(x)\propto (\texttt{score})^{\alpha}$, ${\alpha}$=0.25 & 12.60 $\pm$ 0.63 & 7.84 $\pm$ 0.55 \\
100 & Importance & PRED &$q(x)\propto (\texttt{score})^{\alpha}$, ${\alpha}$=0.5 & 7.90 $\pm$ 0.62 & 12.56 $\pm$ 1.16 \\
100 & SIS & PRED &$q(x)\propto (\texttt{score})^{\alpha}$, ${\alpha}$=0.25 & 11.10 $\pm$ 0.68 & 8.97 $\pm$ 0.70 \\
100 & SIS & PRED &$q(x)\propto (\texttt{score})^{\alpha}$, ${\alpha}$=0.5 & 6.91 $\pm$ 0.43 & 14.36 $\pm$ 1.14 \\
100 & Adaptive stratified & PRED & — & 97.40 $\pm$ 5.54 & 1.02 $\pm$ 0.08 \\
100 & Adaptive IS & PRED &$q(x)\propto (\texttt{score})^{\alpha}$, ${\alpha}$=0.25 & 9.62 $\pm$ 0.65 & 10.31 $\pm$ 0.86 \\
100 & Adaptive IS & PRED &$q(x)\propto (\texttt{score})^{\alpha}$, ${\alpha}$=0.5 & 9.62 $\pm$ 0.65 & 10.31 $\pm$ 0.86 \\
200 & Random Sampling & PRED & — & 48.40 $\pm$ 2.41 & 1.00 $\pm$ 0.00 \\
200 & Stratified & PRED & — & 48.90 $\pm$ 2.40 & 0.99 $\pm$ 0.07 \\
200 & Importance & PRED &$q(x)\propto (\texttt{score})^{\alpha}$, ${\alpha}$=0.25 & 5.99 $\pm$ 0.29 & 8.09 $\pm$ 0.56 \\
200 & Importance & PRED &$q(x)\propto (\texttt{score})^{\alpha}$, ${\alpha}$=0.5 & 3.93 $\pm$ 0.25 & 12.31 $\pm$ 0.99 \\
200 & SIS & PRED &$q(x)\propto (\texttt{score})^{\alpha}$, ${\alpha}$=0.25 & 5.26 $\pm$ 0.25 & 9.20 $\pm$ 0.63 \\
200 & SIS & PRED &$q(x)\propto (\texttt{score})^{\alpha}$, ${\alpha}$=0.5 & 3.52 $\pm$ 0.20 & 13.77 $\pm$ 1.05 \\
200 & Adaptive stratified & PRED & — & 42.30 $\pm$ 1.87 & 1.14 $\pm$ 0.08 \\
200 & Adaptive IS & PRED &$q(x)\propto (\texttt{score})^{\alpha}$, ${\alpha}$=0.25 & 4.37 $\pm$ 0.26 & 11.07 $\pm$ 0.86 \\
200 & Adaptive IS & PRED &$q(x)\propto (\texttt{score})^{\alpha}$, ${\alpha}$=0.5 & 4.37 $\pm$ 0.26 & 11.07 $\pm$ 0.86 \\
300 & Random Sampling & PRED & — & 31.40 $\pm$ 1.42 & 1.00 $\pm$ 0.00 \\
300 & Stratified & PRED & — & 33.60 $\pm$ 1.62 & 0.94 $\pm$ 0.06 \\
300 & Importance & PRED &$q(x)\propto (\texttt{score})^{\alpha}$, ${\alpha}$=0.25 & 3.92 $\pm$ 0.18 & 8.02 $\pm$ 0.52 \\
300 & Importance & PRED &$q(x)\propto (\texttt{score})^{\alpha}$, ${\alpha}$=0.5 & 2.69 $\pm$ 0.14 & 11.68 $\pm$ 0.80 \\
300 & SIS & PRED &$q(x)\propto (\texttt{score})^{\alpha}$, ${\alpha}$=0.25 & 3.22 $\pm$ 0.14 & 9.74 $\pm$ 0.61 \\
300 & SIS & PRED &$q(x)\propto (\texttt{score})^{\alpha}$, ${\alpha}$=0.5 & 2.46 $\pm$ 0.17 & 12.79 $\pm$ 1.03 \\
300 & Adaptive stratified & PRED & — & 29.90 $\pm$ 1.47 & 1.05 $\pm$ 0.07 \\
300 & Adaptive IS & PRED &$q(x)\propto (\texttt{score})^{\alpha}$, ${\alpha}$=0.25 & 2.90 $\pm$ 0.15 & 10.82 $\pm$ 0.74 \\
300 & Adaptive IS & PRED &$q(x)\propto (\texttt{score})^{\alpha}$, ${\alpha}$=0.5 & 2.90 $\pm$ 0.15 & 10.82 $\pm$ 0.74 \\
400 & Random Sampling & PRED & — & 23.30 $\pm$ 1.01 & 1.00 $\pm$ 0.00 \\
400 & Stratified & PRED & — & 25.60 $\pm$ 1.16 & 0.91 $\pm$ 0.06 \\
400 & Importance & PRED &$q(x)\propto (\texttt{score})^{\alpha}$, ${\alpha}$=0.25 & 2.98 $\pm$ 0.14 & 7.80 $\pm$ 0.49 \\
400 & Importance & PRED &$q(x)\propto (\texttt{score})^{\alpha}$, ${\alpha}$=0.5 & 2.04 $\pm$ 0.10 & 11.43 $\pm$ 0.76 \\
400 & SIS & PRED &$q(x)\propto (\texttt{score})^{\alpha}$, ${\alpha}$=0.25 & 2.60 $\pm$ 0.12 & 8.94 $\pm$ 0.57 \\
400 & SIS & PRED &$q(x)\propto (\texttt{score})^{\alpha}$, ${\alpha}$=0.5 & 1.71 $\pm$ 0.08 & 13.62 $\pm$ 0.90 \\
400 & Adaptive stratified & PRED & — & 24.90 $\pm$ 1.25 & 0.93 $\pm$ 0.06 \\
400 & Adaptive IS & PRED &$q(x)\propto (\texttt{score})^{\alpha}$, ${\alpha}$=0.25 & 2.11 $\pm$ 0.11 & 11.03 $\pm$ 0.74 \\
400 & Adaptive IS & PRED &$q(x)\propto (\texttt{score})^{\alpha}$, ${\alpha}$=0.5 & 2.11 $\pm$ 0.11 & 11.03 $\pm$ 0.74 \\
500 & Random Sampling & PRED & — & 19.30 $\pm$ 0.88 & 1.00 $\pm$ 0.00 \\
500 & Stratified & PRED & — & 19.30 $\pm$ 0.92 & 1.00 $\pm$ 0.07 \\
500 & Importance & PRED &$q(x)\propto (\texttt{score})^{\alpha}$, ${\alpha}$=0.25 & 2.39 $\pm$ 0.12 & 8.07 $\pm$ 0.54 \\
500 & Importance & PRED &$q(x)\propto (\texttt{score})^{\alpha}$, ${\alpha}$=0.5 & 1.67 $\pm$ 0.08 & 11.59 $\pm$ 0.76 \\
500 & SIS & PRED &$q(x)\propto (\texttt{score})^{\alpha}$, ${\alpha}$=0.25 & 2.23 $\pm$ 0.10 & 8.66 $\pm$ 0.56 \\
500 & SIS & PRED &$q(x)\propto (\texttt{score})^{\alpha}$, ${\alpha}$=0.5 & 1.40 $\pm$ 0.08 & 13.77 $\pm$ 1.01 \\
500 & Adaptive stratified & PRED & — & 17.30 $\pm$ 0.80 & 1.11 $\pm$ 0.07 \\
500 & Adaptive IS & PRED &$q(x)\propto (\texttt{score})^{\alpha}$, ${\alpha}$=0.25 & 1.72 $\pm$ 0.09 & 11.23 $\pm$ 0.78 \\
500 & Adaptive IS & PRED &$q(x)\propto (\texttt{score})^{\alpha}$, ${\alpha}$=0.5 & 1.72 $\pm$ 0.09 & 11.23 $\pm$ 0.78 \\
600 & Random Sampling & PRED & — & 16.00 $\pm$ 0.70 & 1.00 $\pm$ 0.00 \\
600 & Stratified & PRED & — & 15.90 $\pm$ 0.69 & 1.00 $\pm$ 0.06 \\
600 & Importance & PRED &$q(x)\propto (\texttt{score})^{\alpha}$, ${\alpha}$=0.25 & 2.14 $\pm$ 0.10 & 7.46 $\pm$ 0.47 \\
600 & Importance & PRED &$q(x)\propto (\texttt{score})^{\alpha}$, ${\alpha}$=0.5 & 1.40 $\pm$ 0.07 & 11.41 $\pm$ 0.76 \\
600 & SIS & PRED &$q(x)\propto (\texttt{score})^{\alpha}$, ${\alpha}$=0.25 & 1.72 $\pm$ 0.08 & 9.30 $\pm$ 0.59 \\
600 & SIS & PRED &$q(x)\propto (\texttt{score})^{\alpha}$, ${\alpha}$=0.5 & 1.21 $\pm$ 0.06 & 13.20 $\pm$ 0.88 \\
600 & Adaptive stratified & PRED & — & 14.40 $\pm$ 0.63 & 1.11 $\pm$ 0.07 \\
600 & Adaptive IS & PRED &$q(x)\propto (\texttt{score})^{\alpha}$, ${\alpha}$=0.25 & 1.43 $\pm$ 0.08 & 11.20 $\pm$ 0.81 \\
600 & Adaptive IS & PRED &$q(x)\propto (\texttt{score})^{\alpha}$, ${\alpha}$=0.5 & 1.43 $\pm$ 0.08 & 11.20 $\pm$ 0.81 \\
\bottomrule
\end{tabular}
\caption{Detailed MSE (scaled by $10^{-6}$) and relative efficiency (RE) w.r.t.~Random Sampling for the MNIST dataset. \textbf{Sample sizes from 100 to 600}. Configurations shown use strata on the predicted labels (PRED), with a proposal from the classifier’s maximum softmax probability, over $alpha \in \{0.25, 0.5\}$. Values are mean $\pm$ SE over simulations.}
\label{tab:mnist_detailed_results_grid_table_100_700}
\end{table*}

\begin{table*}[t]
\centering
\small
\begin{tabular}{rllllll}
\toprule
\makecell{Sample \\ size} & Sample design & Strata (if any) & Proposal (if any) & MSE ($\times 10^{-6}$) & \makecell{RE w.r.t.\\ Random} \\
\midrule
800 & Random Sampling & PRED & — & 12.90 $\pm$ 0.59 & 1.00 $\pm$ 0.00 \\
800 & Stratified & PRED & — & 12.50 $\pm$ 0.57 & 1.03 $\pm$ 0.07 \\
800 & Importance & PRED &$q(x)\propto (\texttt{score})^{\alpha}$, ${\alpha}$=0.25 & 1.46 $\pm$ 0.07 & 8.80 $\pm$ 0.57 \\
800 & SIS & PRED &$q(x)\propto (\texttt{score})^{\alpha}$, ${\alpha}$=0.25 & 1.19 $\pm$ 0.06 & 10.79 $\pm$ 0.71 \\
800 & Adaptive stratified & PRED & — & 11.10 $\pm$ 0.46 & 1.16 $\pm$ 0.07 \\
800 & Adaptive IS & PRED &$q(x)\propto (\texttt{score})^{\alpha}$, ${\alpha}$=0.25 & 0.99 $\pm$ 0.05 & 12.94 $\pm$ 0.86 \\
900 & Random Sampling & PRED & — & 10.40 $\pm$ 0.48 & 1.00 $\pm$ 0.00 \\
900 & Stratified & PRED & — & 10.40 $\pm$ 0.48 & 1.01 $\pm$ 0.07 \\
900 & Importance & PRED &$q(x)\propto (\texttt{score})^{\alpha}$, ${\alpha}$=0.25 & 1.37 $\pm$ 0.07 & 7.60 $\pm$ 0.51 \\
900 & SIS & PRED &$q(x)\propto (\texttt{score})^{\alpha}$, ${\alpha}$=0.25 & 1.19 $\pm$ 0.06 & 8.77 $\pm$ 0.57 \\
900 & Adaptive stratified & PRED & — & 9.63 $\pm$ 0.44 & 1.08 $\pm$ 0.07 \\
900 & Adaptive IS & PRED &$q(x)\propto (\texttt{score})^{\alpha}$, ${\alpha}$=0.25 & 0.85 $\pm$ 0.04 & 12.24 $\pm$ 0.81 \\
1000 & Random Sampling & PRED & — & 9.67 $\pm$ 0.49 & 1.00 $\pm$ 0.00 \\
1000 & Stratified & PRED & — & 9.90 $\pm$ 0.45 & 0.98 $\pm$ 0.07 \\
1000 & Importance & PRED &$q(x)\propto (\texttt{score})^{\alpha}$, ${\alpha}$=0.25 & 1.18 $\pm$ 0.05 & 8.19 $\pm$ 0.56 \\
1000 & SIS & PRED &$q(x)\propto (\texttt{score})^{\alpha}$, ${\alpha}$=0.25 & 1.17 $\pm$ 0.05 & 8.27 $\pm$ 0.55 \\
1000 & Adaptive stratified & PRED & — & 9.36 $\pm$ 0.39 & 1.03 $\pm$ 0.07 \\
1000 & Adaptive IS & PRED &$q(x)\propto (\texttt{score})^{\alpha}$, ${\alpha}$=0.25 & 0.84 $\pm$ 0.04 & 11.50 $\pm$ 0.77 \\
1200 & Random Sampling & PRED & — & 8.43 $\pm$ 0.40 & 1.00 $\pm$ 0.00 \\
1200 & Stratified & PRED & — & 8.46 $\pm$ 0.37 & 1.00 $\pm$ 0.06 \\
1200 & Importance & PRED &$q(x)\propto (\texttt{score})^{\alpha}$, ${\alpha}$=0.25 & 1.04 $\pm$ 0.05 & 8.08 $\pm$ 0.52 \\
1200 & SIS & PRED &$q(x)\propto (\texttt{score})^{\alpha}$, ${\alpha}$=0.25 & 0.88 $\pm$ 0.04 & 9.63 $\pm$ 0.62 \\
1200 & Adaptive stratified & PRED & — & 7.46 $\pm$ 0.33 & 1.13 $\pm$ 0.07 \\
1200 & Adaptive IS & PRED &$q(x)\propto (\texttt{score})^{\alpha}$, ${\alpha}$=0.25 & 0.62 $\pm$ 0.03 & 13.48 $\pm$ 0.87 \\
1400 & Random Sampling & PRED & — & 7.01 $\pm$ 0.30 & 1.00 $\pm$ 0.00 \\
1400 & Stratified & PRED & — & 7.12 $\pm$ 0.33 & 0.99 $\pm$ 0.06 \\
1400 & Importance & PRED &$q(x)\propto (\texttt{score})^{\alpha}$, ${\alpha}$=0.25 & 0.90 $\pm$ 0.04 & 7.82 $\pm$ 0.48 \\
1400 & SIS & PRED &$q(x)\propto (\texttt{score})^{\alpha}$, ${\alpha}$=0.25 & 0.76 $\pm$ 0.03 & 9.27 $\pm$ 0.58 \\
1400 & Adaptive stratified & PRED & — & 6.55 $\pm$ 0.30 & 1.07 $\pm$ 0.07 \\
1400 & Adaptive IS & PRED &$q(x)\propto (\texttt{score})^{\alpha}$, ${\alpha}$=0.25 & 0.55 $\pm$ 0.03 & 12.82 $\pm$ 0.82 \\
\bottomrule
\end{tabular}
\caption{Detailed MSE (scaled by $10^{-6}$) and relative efficiency (RE) w.r.t.~Random Sampling for the MNIST dataset. \textbf{Sample sizes from 800 to 1,400}. Configurations shown use strata on the predicted labels (PRED), with a proposal from the classifier’s maximum softmax probability, over $alpha \in \{0.25, 0.5\}$. Values are mean $\pm$ SE over simulations.}
\label{tab:mnist_detailed_results_grid_table_800_1400}
\end{table*}

\subsubsection{CIFAR10}
Table~\ref{tab:cifar10_detailed_results_grid} summarizes the experimental results on the CIFAR10 dataset. 
This multiclass image classification task presents a substantially larger defect rate (approximately $18.8\%$). 
Across sample sizes, Stratified Importance Sampling (SIS) delivers consistent efficiency gains over Random Sampling (RS) and Stratified Random Sampling (SRS), achieving relative efficiencies around $1.6\times$ at $n = 200$. 
While standard Importance Sampling (IS) can surpass SIS at higher labeling budgets (e.g., at $n = 1,000$), SIS maintains more stable behavior across budgets, mitigating variance spikes from proposal mismatch. 

\begin{table*}[t]
\centering
\small
\begin{tabular}{rllllll}
\toprule
\makecell{Sample \\ size} & Sample design & Strata (if any) & Proposal (if any) & MSE ($\times 10^{-4}$) & \makecell{RE w.r.t.\\ Random} \\
\midrule
200 & Random Sampling & — & — & 7.71 $\pm$ 0.58 & — \\
200 & Stratified & PRED & — & 7.77 $\pm$ 0.68 & 0.99 $\pm$ 0.11 \\
200 & Stratified & BR\_BIN & — & 7.66 $\pm$ 0.58 & 1.01 $\pm$ 0.11 \\
200 & Importance & — & $q(x)\propto(\texttt{score})^{\alpha}$, $\alpha=0.5$ & 6.89 $\pm$ 1.87 & 1.12 $\pm$ 0.31 \\
200 & SIS & PRED & $q(x)\propto(\texttt{score})^{\alpha}$, $\alpha=0.5$ & 5.11 $\pm$ 0.44 & 1.51 $\pm$ 0.17 \\
200 & SIS & BR\_BIN & $q(x)\propto(\texttt{score})^{\alpha}$, $\alpha=0.5$ & 4.63 $\pm$ 0.36 & 1.66 $\pm$ 0.18 \\
300 & Random Sampling & — & — & 5.17 $\pm$ 0.40 & — \\
300 & Stratified & PRED & — & 4.88 $\pm$ 0.35 & 0.95 $\pm$ 0.11 \\
300 & Stratified & BR\_BIN & — & 5.42 $\pm$ 0.44 & 1.06 $\pm$ 0.11 \\
300 & Importance & — & $q(x)\propto(\texttt{score})^{\alpha}$, $\alpha=0.5$ & 2.99 $\pm$ 0.23 & 1.73 $\pm$ 0.19 \\
300 & SIS & PRED & $q(x)\propto(\texttt{score})^{\alpha}$, $\alpha=0.5$ & 2.91 $\pm$ 0.26 & 1.40 $\pm$ 0.19 \\
300 & SIS & BR\_BIN & $q(x)\propto(\texttt{score})^{\alpha}$, $\alpha=0.5$ & 3.70 $\pm$ 0.41 & 1.78 $\pm$ 0.21 \\
400 & Random Sampling & — & — & 4.05 $\pm$ 0.32 & — \\
400 & Stratified & PRED & — & 3.38 $\pm$ 0.30 & 1.00 $\pm$ 0.11 \\
400 & Stratified & BR\_BIN & — & 4.07 $\pm$ 0.34 & 1.20 $\pm$ 0.14 \\
400 & Importance & — & $q(x)\propto(\texttt{score})^{\alpha}$, $\alpha=0.5$ & 2.39 $\pm$ 0.20 & 1.69 $\pm$ 0.19 \\
400 & SIS & PRED & $q(x)\propto(\texttt{score})^{\alpha}$, $\alpha=0.5$ & 2.32 $\pm$ 0.19 & 1.37 $\pm$ 0.25 \\
400 & SIS & BR\_BIN & $q(x)\propto(\texttt{score})^{\alpha}$, $\alpha=0.5$ & 2.96 $\pm$ 0.49 & 1.75 $\pm$ 0.20 \\
500 & Random Sampling & — & — & 3.22 $\pm$ 0.28 & — \\
500 & Stratified & PRED & — & 2.75 $\pm$ 0.22 & 1.04 $\pm$ 0.12 \\
500 & Stratified & BR\_BIN & — & 3.10 $\pm$ 0.24 & 1.17 $\pm$ 0.14 \\
500 & Importance & — & $q(x)\propto(\texttt{score})^{\alpha}$, $\alpha=0.5$ & 2.44 $\pm$ 0.29 & 1.32 $\pm$ 0.19 \\
500 & SIS & PRED & $q(x)\propto(\texttt{score})^{\alpha}$, $\alpha=0.5$ & 1.95 $\pm$ 0.18 & 1.44 $\pm$ 0.18 \\
500 & SIS & BR\_BIN & $q(x)\propto(\texttt{score})^{\alpha}$, $\alpha=0.5$ & 2.23 $\pm$ 0.21 & 1.65 $\pm$ 0.21 \\
750 & Random Sampling & — & — & 2.31 $\pm$ 0.18 & — \\
750 & Stratified & PRED & — & 1.86 $\pm$ 0.15 & 1.24 $\pm$ 0.14 \\
750 & Stratified & BR\_BIN & — & 1.91 $\pm$ 0.14 & 1.21 $\pm$ 0.13 \\
750 & Importance & — & $q(x)\propto(\texttt{score})^{\alpha}$, $\alpha=0.5$ & 1.42 $\pm$ 0.16 & 1.62 $\pm$ 0.22 \\
750 & SIS & PRED & $q(x)\propto(\texttt{score})^{\alpha}$, $\alpha=0.5$ & 1.15 $\pm$ 0.09 & 2.01 $\pm$ 0.22 \\
750 & SIS & BR\_BIN & $q(x)\propto(\texttt{score})^{\alpha}$, $\alpha=0.5$ & 1.44 $\pm$ 0.14 & 1.59 $\pm$ 0.19 \\
1000 & Random Sampling & — & — & 1.88 $\pm$ 0.15 & — \\
1000 & Stratified & PRED & — & 1.43 $\pm$ 0.10 & 1.26 $\pm$ 0.15 \\
1000 & Stratified & BR\_BIN & — & 1.49 $\pm$ 0.13 & 1.31 $\pm$ 0.14 \\
1000 & Importance & — & $q(x)\propto(\texttt{score})^{\alpha}$, $\alpha=0.5$ & 1.09 $\pm$ 0.14 & 1.72 $\pm$ 0.25 \\
1000 & SIS & PRED & $q(x)\propto(\texttt{score})^{\alpha}$, $\alpha=0.5$ & 1.13 $\pm$ 0.15 & 1.66 $\pm$ 0.26 \\
1000 & SIS & BR\_BIN & $q(x)\propto(\texttt{score})^{\alpha}$, $\alpha=0.5$ & 1.23 $\pm$ 0.15 & 1.53 $\pm$ 0.23 \\
\bottomrule
\end{tabular}
\caption{Detailed MSE (scaled by $10^{-4}$) and relative efficiency (RE) w.r.t. Random Sampling for the CIFAR10 dataset. Configurations shown use strata on the predicted labels (PRED) and brightness quintiles (BR\_BIN), with a proposal proportional to the classifier's maximum softmax probability, with $\alpha = 0.5$. Values are mean $\pm$ SE over simulations.}
\label{tab:cifar10_detailed_results_grid}
\end{table*}

\subsubsection{Proprietary dataset}
Table \ref{tab:proprietary_detailed_results_grid} reports the results on the Proprietary dataset, representing a large-scale, real-world binary classification deployment with a near-perfect accuracy of $99.56\%$ and a defect rate below $0.5\%$. 
In this extremely low-error regime, Stratified Importance Sampling (SIS) substantially improves estimation efficiency relative to Random Sampling (RS), achieving relative efficiencies above $6\times$ at $n = 500$. 
Compared to Importance Sampling (IS), SIS provides more stable performance, especially at smaller budgets where IS exhibits higher estimator variance due to extreme weighting. 
At larger budgets (e.g., $n = 2500$), SIS remains competitive, though IS may slightly outperform it when proposals are well aligned with the true error distribution. 

\begin{table*}[t]
\centering
\small
\begin{tabular}{rllllll}
\toprule
\makecell{Sample \\ size} & Sample design & Strata (if any) & Proposal (if any) & MSE ($\times 10^{-6}$) & \makecell{RE w.r.t.\\ Random} \\
\midrule
250 & Random Sampling & — & — & 18.38 $\pm$ 1.07 & — \\
250 & Stratified &  PRED x ModelVersion & — & 6.10 $\pm$ 0.29 & 3.01 $\pm$ 0.23 \\
250 & Importance & — & $q(x) \propto (\texttt{score})^{\alpha}$, $\alpha=0.25$ & 5.83 $\pm$ 0.30 & 3.15 $\pm$ 0.24 \\
250 & Importance & — & $q(x) \propto (\texttt{score})^{\alpha}$, $\alpha=0.5$ & 4.51 $\pm$ 1.70 & 4.07 $\pm$ 1.56 \\
250 & SIS &  PRED x ModelVersion & $q(x) \propto (\texttt{score})^{\alpha}$, $\alpha=0.25$ & 3.38 $\pm$ 0.21 & 5.44 $\pm$ 0.46 \\
250 & SIS &  PRED x ModelVersion & $q(x) \propto (\texttt{score})^{\alpha}$, $\alpha=0.5$ & 2.48 $\pm$ 0.28 & 7.41 $\pm$ 0.95 \\
500 & Random Sampling & — & — & 8.53 $\pm$ 0.42 & — \\
500 & Stratified &  PRED x ModelVersion & — & 3.31 $\pm$ 0.14 & 2.57 $\pm$ 0.17 \\
500 & Importance & — & $q(x) \propto (\texttt{score})^{\alpha}$, $\alpha=0.25$ & 3.05 $\pm$ 0.14 & 2.79 $\pm$ 0.19 \\
500 & Importance & — & $q(x) \propto (\texttt{score})^{\alpha}$, $\alpha=0.5$ & 2.04 $\pm$ 0.50 & 4.18 $\pm$ 1.04 \\
500 & SIS &  PRED x ModelVersion & $q(x) \propto (\texttt{score})^{\alpha}$, $\alpha=0.25$ & 1.81 $\pm$ 0.09 & 4.72 $\pm$ 0.34 \\
500 & SIS &  PRED x ModelVersion & $q(x) \propto (\texttt{score})^{\alpha}$, $\alpha=0.5$ & 1.35 $\pm$ 0.07 & 6.32 $\pm$ 0.46 \\
750 & Random Sampling & — & — & 5.59 $\pm$ 0.27 & — \\
750 & Stratified &  PRED x ModelVersion & — & 2.31 $\pm$ 0.09 & 2.42 $\pm$ 0.15 \\
750 & Importance & — & $q(x) \propto (\texttt{score})^{\alpha}$, $\alpha=0.25$ & 2.15 $\pm$ 0.11 & 2.60 $\pm$ 0.18 \\
750 & Importance & — & $q(x) \propto (\texttt{score})^{\alpha}$, $\alpha=0.5$ & 1.33 $\pm$ 0.27 & 4.21 $\pm$ 0.87 \\
750 & SIS &  PRED x ModelVersion & $q(x) \propto (\texttt{score})^{\alpha}$, $\alpha=0.25$ & 1.29 $\pm$ 0.07 & 4.33 $\pm$ 0.30 \\
750 & SIS &  PRED x ModelVersion & $q(x) \propto (\texttt{score})^{\alpha}$, $\alpha=0.5$ & 1.01 $\pm$ 0.13 & 5.55 $\pm$ 0.76 \\
1000 & Random Sampling & — & — & 4.21 $\pm$ 0.20 & — \\
1000 & Stratified &  PRED x ModelVersion & — & 1.92 $\pm$ 0.08 & 2.19 $\pm$ 0.14 \\
1000 & Importance & — & $q(x) \propto (\texttt{score})^{\alpha}$, $\alpha=0.25$ & 1.50 $\pm$ 0.07 & 2.82 $\pm$ 0.19 \\
1000 & Importance & — & $q(x) \propto (\texttt{score})^{\alpha}$, $\alpha=0.5$ & 0.97 $\pm$ 0.16 & 4.34 $\pm$ 0.73 \\
1000 & SIS &  PRED x ModelVersion & $q(x) \propto (\texttt{score})^{\alpha}$, $\alpha=0.25$ & 0.95 $\pm$ 0.04 & 4.44 $\pm$ 0.29 \\
1000 & SIS &  PRED x ModelVersion & $q(x) \propto (\texttt{score})^{\alpha}$, $\alpha=0.5$ & 0.73 $\pm$ 0.04 & 5.75 $\pm$ 0.39 \\
1250 & Random Sampling & — & — & 3.38 $\pm$ 0.17 & — \\
1250 & Stratified &  PRED x ModelVersion & — & 1.44 $\pm$ 0.06 & 2.34 $\pm$ 0.15 \\
1250 & Importance & — & $q(x) \propto (\texttt{score})^{\alpha}$, $\alpha=0.25$ & 1.19 $\pm$ 0.05 & 2.83 $\pm$ 0.19 \\
1250 & Importance & — & $q(x) \propto (\texttt{score})^{\alpha}$, $\alpha=0.5$ & 0.65 $\pm$ 0.10 & 5.21 $\pm$ 0.87 \\
1250 & SIS &  PRED x ModelVersion & $q(x) \propto (\texttt{score})^{\alpha}$, $\alpha=0.25$ & 0.77 $\pm$ 0.03 & 4.36 $\pm$ 0.28 \\
1250 & SIS &  PRED x ModelVersion & $q(x) \propto (\texttt{score})^{\alpha}$, $\alpha=0.5$ & 0.61 $\pm$ 0.02 & 5.50 $\pm$ 0.34 \\
1500 & Random Sampling & — & — & 2.90 $\pm$ 0.14 & — \\
1500 & Stratified &  PRED x ModelVersion & — & 1.26 $\pm$ 0.05 & 2.31 $\pm$ 0.15 \\
1500 & Importance & — & $q(x) \propto (\texttt{score})^{\alpha}$, $\alpha=0.25$ & 1.03 $\pm$ 0.05 & 2.81 $\pm$ 0.19 \\
1500 & Importance & — & $q(x) \propto (\texttt{score})^{\alpha}$, $\alpha=0.5$ & 0.68 $\pm$ 0.17 & 4.25 $\pm$ 1.05 \\
1500 & SIS &  PRED x ModelVersion & $q(x) \propto (\texttt{score})^{\alpha}$, $\alpha=0.25$ & 0.65 $\pm$ 0.03 & 4.42 $\pm$ 0.29 \\
1500 & SIS &  PRED x ModelVersion & $q(x) \propto (\texttt{score})^{\alpha}$, $\alpha=0.5$ & 0.49 $\pm$ 0.02 & 5.91 $\pm$ 0.36 \\
1750 & Random Sampling & — & — & 2.38 $\pm$ 0.11 & — \\
1750 & Stratified &  PRED x ModelVersion & — & 1.06 $\pm$ 0.04 & 2.26 $\pm$ 0.14 \\
1750 & Importance & — & $q(x) \propto (\texttt{score})^{\alpha}$, $\alpha=0.25$ & 0.88 $\pm$ 0.04 & 2.71 $\pm$ 0.17 \\
1750 & Importance & — & $q(x) \propto (\texttt{score})^{\alpha}$, $\alpha=0.5$ & 0.40 $\pm$ 0.02 & 5.96 $\pm$ 0.38 \\
1750 & SIS &  PRED x ModelVersion & $q(x) \propto (\texttt{score})^{\alpha}$, $\alpha=0.25$ & 0.55 $\pm$ 0.02 & 4.34 $\pm$ 0.25 \\
1750 & SIS &  PRED x ModelVersion & $q(x) \propto (\texttt{score})^{\alpha}$, $\alpha=0.5$ & 0.54 $\pm$ 0.05 & 4.38 $\pm$ 0.45 \\
2000 & Random Sampling & — & — & 2.08 $\pm$ 0.09 & — \\
2000 & Stratified &  PRED x ModelVersion & — & 0.92 $\pm$ 0.04 & 2.27 $\pm$ 0.13 \\
2000 & Importance & — & $q(x) \propto (\texttt{score})^{\alpha}$, $\alpha=0.25$ & 0.77 $\pm$ 0.04 & 2.72 $\pm$ 0.18 \\
2000 & Importance & — & $q(x) \propto (\texttt{score})^{\alpha}$, $\alpha=0.5$ & 0.51 $\pm$ 0.09 & 4.06 $\pm$ 0.75 \\
2000 & SIS &  PRED x ModelVersion & $q(x) \propto (\texttt{score})^{\alpha}$, $\alpha=0.25$ & 0.55 $\pm$ 0.02 & 3.77 $\pm$ 0.22 \\
2000 & SIS &  PRED x ModelVersion & $q(x) \propto (\texttt{score})^{\alpha}$, $\alpha=0.5$ & 0.42 $\pm$ 0.01 & 4.96 $\pm$ 0.27 \\
2250 & Random Sampling & — & — & 1.96 $\pm$ 0.09 & — \\
2250 & Stratified &  PRED x ModelVersion & — & 0.83 $\pm$ 0.03 & 2.37 $\pm$ 0.15 \\
2250 & Importance & — & $q(x) \propto (\texttt{score})^{\alpha}$, $\alpha=0.25$ & 0.70 $\pm$ 0.03 & 2.80 $\pm$ 0.19 \\
2250 & Importance & — & $q(x) \propto (\texttt{score})^{\alpha}$, $\alpha=0.5$ & 0.45 $\pm$ 0.09 & 4.41 $\pm$ 0.87 \\
2250 & SIS &  PRED x ModelVersion & $q(x) \propto (\texttt{score})^{\alpha}$, $\alpha=0.25$ & 0.51 $\pm$ 0.02 & 3.88 $\pm$ 0.24 \\
2250 & SIS &  PRED x ModelVersion & $q(x) \propto (\texttt{score})^{\alpha}$, $\alpha=0.5$ & 0.46 $\pm$ 0.04 & 4.23 $\pm$ 0.42 \\
2500 & Random Sampling & — & — & 1.67 $\pm$ 0.07 & — \\
2500 & Stratified &  PRED x ModelVersion & — & 0.69 $\pm$ 0.03 & 2.43 $\pm$ 0.15 \\
2500 & Importance & — & $q(x) \propto (\texttt{score})^{\alpha}$, $\alpha=0.25$ & 0.60 $\pm$ 0.03 & 2.77 $\pm$ 0.17 \\
2500 & Importance & — & $q(x) \propto (\texttt{score})^{\alpha}$, $\alpha=0.5$ & 0.34 $\pm$ 0.03 & 4.86 $\pm$ 0.53 \\
2500 & SIS &  PRED x ModelVersion & $q(x) \propto (\texttt{score})^{\alpha}$, $\alpha=0.25$ & 0.36 $\pm$ 0.01 & 4.66 $\pm$ 0.28 \\
2500 & SIS &  PRED x ModelVersion & $q(x) \propto (\texttt{score})^{\alpha}$, $\alpha=0.5$ & 0.37 $\pm$ 0.07 & 4.53 $\pm$ 0.92 \\
\bottomrule
\end{tabular}
\caption{Detailed MSE (scaled by $10^{-6}$) and relative efficiency (RE) w.r.t. Random Sampling for the proprietary dataset. The strata correspond to  PRED x ModelVersion, with proposals proportional to $(\texttt{score})^{\alpha}$ for $\alpha \in \{0.25, 0.5\}$. Values are mean $\pm$ SE over simulations.}
\label{tab:proprietary_detailed_results_grid}
\end{table*}

\vfill

\end{document}